\documentclass[a4paper]{cas-sc}

\usepackage{amsthm}
\usepackage{amsmath}
\usepackage{amssymb}
\usepackage{placeins}
\usepackage{float}

\newtheoremstyle{casestyle}
{\topsep}{\topsep}{\itshape}{}{\bfseries}{.}{.5em}{}
\theoremstyle{casestyle}
\newtheorem{assumption}{Assumption}[section]
\newtheorem{proposition}{Proposition}[section]
\newtheorem{lemma}{Lemma}[section]
\newtheorem{theorem}{Theorem}[section]
\newtheorem{corollary}{Corollary}[section]
\newtheorem{remark}{Remark}[section]

\DeclareMathOperator*{\vm}{m}
\newcommand{\R}{\mathbb R}
\newcommand{\cA}{\mathcal A}
\newcommand{\cD}{\mathcal D}
\newcommand{\cH}{\mathcal H}
\newcommand{\cK}{\mathcal K}
\newcommand{\cL}{\mathcal L}
\newcommand{\cN}{\mathcal N}
\newcommand{\cE}{\mathcal E}
\newcommand{\cX}{\mathcal X}

\newcommand{\col}{\operatorname{col}}
\usepackage{algorithm}
\usepackage{algorithmic}
\usepackage[numbers,sort&compress]{natbib}
\usepackage{pifont}
\usepackage{cleveref}

\def\tsc#1{\csdef{#1}{\textsc{\lowercase{#1}}\xspace}}
\tsc{WGM}
\tsc{QE}
\tsc{EP}
\tsc{PMS}
\tsc{BEC}
\tsc{DE}

\begin{document}
\let\WriteBookmarks\relax
\def\floatpagepagefraction{1}
\def\textpagefraction{.001}
\shorttitle{Manifold-orthogonal dual-spectrum extrapolation}
\shortauthors{Z.Y. Liang et~al.}

\title [mode = title]{Manifold-orthogonal dual-spectrum extrapolation for parameterized physics-informed neural networks}

\tnotemark[1,2]

\author[1]{Zhangyong Liang}

\author[2]{Huanhuan Gao}[orcid=0000-0003-4463-6433]
\cormark[1]
\ead{gao_huanhuan@jlu.edu.cn}

\affiliation[1]{organization={National Center for Applied Mathematics, Tianjin University},
                addressline={Weijin Road 92},
                postcode={300072},
                city={Tianjin},
                country={China}}

\affiliation[2]{organization={School of Mechanical and Aerospace Engineering, Jilin University},
                addressline={Renmin Street 5988},
                postcode={130025},
                city={Changchun},
                country={China}}

\cortext[cor1]{Corresponding author}

\begin{abstract}
Physics-informed neural networks (PINNs) have made significant strides in modeling dynamical systems governed by partial differential equations (PDEs). 
To avoid computationally prohibitive retraining for new physical conditions, parameterized PINNs (P$^2$INNs) typically employ singular value decomposition (SVD) to adapt pre-trained operators to out-of-distribution (OOD) regimes. 
However, native SVD fine-tuning suffers from rigid subspace locking and the truncation of critical high-frequency spectral modes, ultimately failing to capture complex physical phase transitions. 
While parameter-efficient fine-tuning (PEFT) methods naturally emerge as potential alternatives to resolve these limitations, applying conventional PEFT adapters (e.g., LoRA) to P$^2$INNs triggers a severe Pareto disaster. 
These additive adapters incur immense parameter bloat and recklessly distort established physical manifolds due to the rigid, multiscale nature of physical landscapes. 
To shatter these theoretical deadlocks, we propose Manifold-Orthogonal Dual-spectrum Extrapolation (MODE), a foundational micro-architecture exclusively tailored for physical operator adaptation. 
MODE fundamentally resolves the Pareto dilemma by strategically decoupling physical evolution into three mechanisms. 
The first is principal-spectrum dense mixing, which enables cross-modal energy transfer within frozen orthogonal bases at negligible parameter cost. 
The second is residual-spectrum awakening, a pioneering dual-spectrum approach that dynamically ignites dormant high-frequency shockwaves via a single trainable scalar. 
The third is affine Galilean unlocking, which explicitly isolates kinematic spatial translations. 
Extensive experiments on challenging 1D Convection-Diffusion-Reaction and 2D Helmholtz equations demonstrate that MODE establishes absolute Pareto dominance. 
It strictly compresses spatial complexity to the theoretical minimum of native SVD while decisively surpassing the OOD generalization fidelity of all existing state-of-the-art PEFT baselines.
\end{abstract}

\begin{graphicalabstract}
\includegraphics{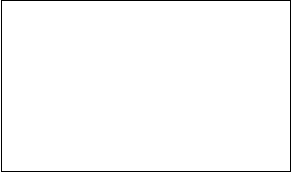}
\end{graphicalabstract}

\begin{highlights}
\item MODE enables parameter-efficient out-of-distribution P$^2$INN adaptation.
\item Principal-spectrum dense mixing unlocks cross-modal transfer in frozen bases.
\item Residual-spectrum awakening revives truncated high-frequency PDE modes.
\item Affine Galilean unlocking decouples translations from weight evolution.
\item MODE resolves the PEFT Pareto dilemma via compact dual-spectrum adaptation.
\end{highlights}


\maketitle

\section{Introduction}
\label{sec:introduction}

Scientific machine learning (SML)~\cite{baker2019workshop} has been growing fast.
Unlike traditional tasks in machine learning, e.g., image classification and natural language processing (NLP), SML requires the exact satisfaction of complex physical characteristics.
Recent work has developed various deep-learning models that encode such physical characteristics, making them physically-consistent (e.g., by enforcing conservation laws~\cite{raissi2019physics, lee2021deep}) and preserving structures~\cite{Greydanus2019hnn, toth2019hamiltonian}.
Among those methods, physics-informed neural networks (PINNs)~\cite{raissi2019physics, karniadakis2021physics} are gaining immense traction because of their sound computational formalism to enforce governing physical laws to learn continuous solutions. 

While powerful, vanilla PINNs suffer from obvious weaknesses: PDE operators are highly nonlinear, making optimization extremely difficult~\cite{krishnapriyan2021characterizing, wang2021understanding}, and repetitive, time-consuming training from scratch is required when evaluating solutions for new PDE parameters (e.g., the convective coefficient $\beta$ or reaction rate $\rho$).
To mitigate these issues, parameterized physics-informed neural networks (P$^2$INNs) have been proposed to explicitly encode PDE parameters into a latent manifold, serving as a foundational model capable of simultaneously learning families of parameterized PDEs.
To adapt pre-trained P$^2$INNs to out-of-distribution (OOD) physical parameters efficiently, Singular Value Decomposition (SVD) modulation~\cite{Gao_2022} is often employed to fine-tune only the core components of the network.

Concurrently, Parameter-Efficient Fine-Tuning (PEFT)~\cite{peft} has achieved unprecedented success in adapting large pre-trained models~\cite{gpt2, T5, llama2, llama3, gemma} to downstream tasks.
Existing PEFT methods have become the \textit{de facto} standard, ranging from additive low-rank adapters like LoRA~\cite{lora}, AdaLoRA~\cite{adalora}, DoRA~\cite{dora}, ReLoRA~\cite{relora}, and VeRA~\cite{vera}, to spectral and orthogonal fine-tuning methods like OFT~\cite{oft}, BOFT~\cite{boft}, and PiSSA~\cite{pissa}, as well as bias-only or sparse updates like BitFit~\cite{bitfit} and AdapterDrop~\cite{adapterdrop}.
Driven by this explosive success, a natural intuition is to directly port these cutting-edge PEFT methods to P$^2$INNs, replacing the native SVD modulation for OOD physical parameter adaptation.

\begin{figure}
\centering
\includegraphics[width=.9\textwidth]{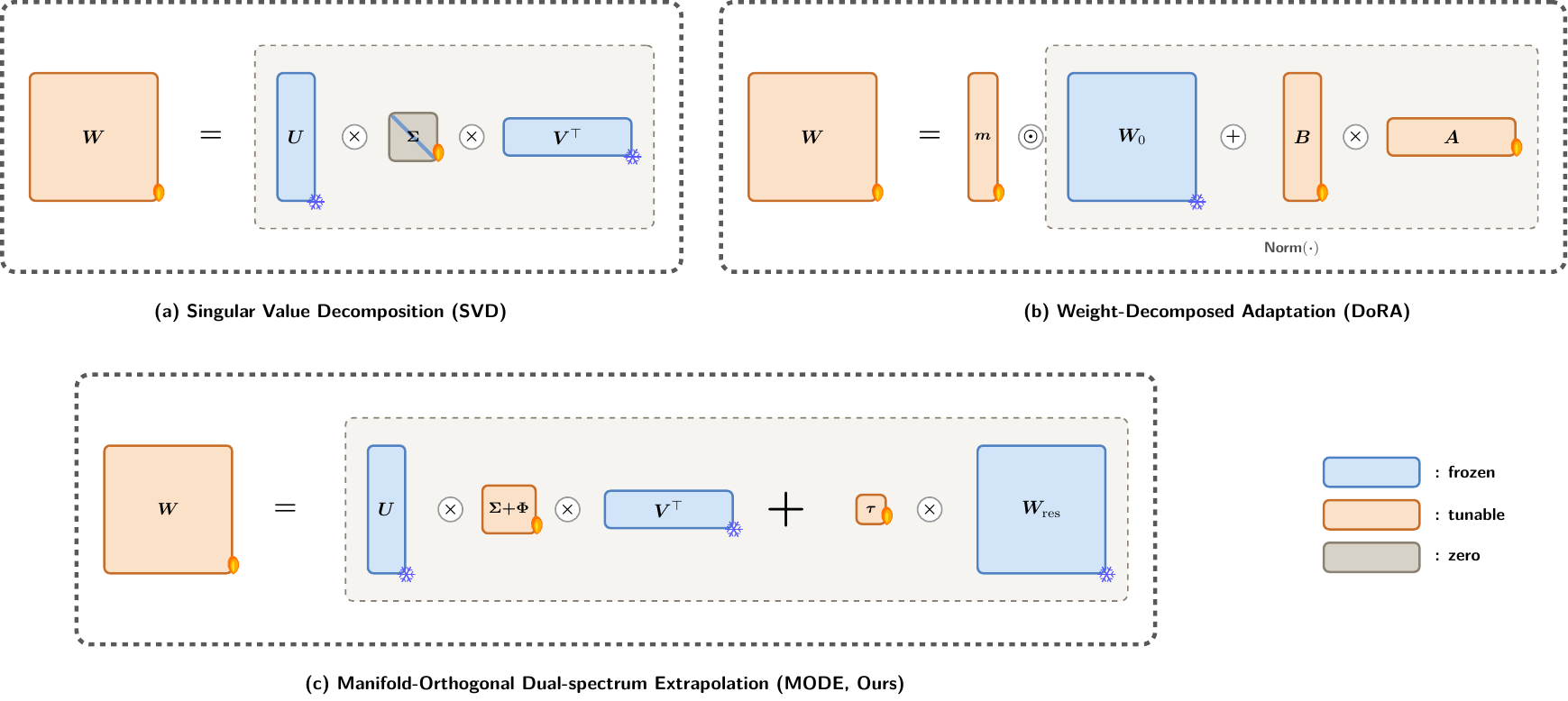} 
\caption{Comparison of different fine-tuning approaches for P$^2$INNs. (a) Native Singular Value Decomposition (SVD), (b) Weight-Decomposed Adaptation (DoRA), and (c) our proposed Manifold-Orthogonal Dual-spectrum Extrapolation (MODE). (\textcolor{blue}{\ding{110}}: frozen parameters; \textcolor{orange}{\ding{110}}: tunable parameters; \textcolor{gray}{\ding{110}}: zeros.)}
\label{fig.compare} 
\end{figure}
\FloatBarrier

However, our empirical investigations reveal a striking incompatibility: directly migrating these existing PEFT methods to P$^2$INNs for physical PDE solving tasks fails to deliver acceptable performance and inevitably leads to optimization collapse.
The fundamental reason lies in a profound domain mismatch.
Conventional PEFT typically operates on discrete tokens within relatively smooth cross-entropy optimization landscapes that exhibit low intrinsic dimensionality~\cite{aghajanyan-etal-2021-intrinsic, li2018measuring}.
In stark contrast, PDE solving requires exactly satisfying high-order spatial-temporal derivatives and stiff physical gradients, resulting in a notoriously non-convex, rigid, and multiscale loss landscape~\cite{krishnapriyan2021characterizing}.
Consequently, replacing SVD with these PEFT triggers a severe Pareto Disaster in scientific computing, exposing five fundamental deadlocks:

\begin{itemize}
    \item \textbf{The Pareto disaster in fine-tuning:} Existing PEFT methods face a severe Pareto dilemma in solving PDEs.
    They either maintain a minimal memory footprint but suffer from unacceptable extrapolation errors (e.g., native SVD~\cite{Gao_2022}), or incur an exorbitant parameter bloat to marginally reduce errors (e.g., DoRA~\cite{dora}, Spectral adapters~\cite{oft}), entirely failing to approach the fidelity of full fine-tuning.
    \item \textbf{Parameter bloat via ambient space search:} Additive adapters (e.g., LoRA~\cite{lora}, AdaLoRA~\cite{adalora}) conduct blind geometric searches in the massive, unconstrained ambient space ($\mathbb{R}^{d_{out} \times d_{in}}$).
    This recklessly discards the orthogonal physical priors crystallized during pre-training, causing severe $\mathcal{O}(d \cdot k)$ parameter redundancy and distorting the underlying physical manifold under rigid PDE constraints.
    \item \textbf{Subspace locking in native SVF:} To strictly compress parameters, native Singular Value Fine-Tuning~\cite{Gao_2022} and principal adaptations~\cite{pissa} update only the diagonal singular values.
    This ``diagonal rigidity'' fundamentally prohibits cross-modal combinations, ``locking'' the subspace and rendering the network incapable of rotating feature lines to fit macroscopic convective phase-shifts.
    \item \textbf{The single-spectrum truncation trap:} Advanced orthogonal adapters (e.g., OFT~\cite{oft}, BOFT~\cite{boft}) universally fail to break the $\sim 10^0$ error barrier because they permanently discard the truncated residual spectrum as noise.
    In physical PDEs, these tiny tail singular values densely encode critical \textit{high-frequency physical modes} (e.g., shocks, steep gradients).
    Discarding them acts as a hard low-pass filter, collapsing the model's high-frequency expressivity during OOD extrapolation.
    \item \textbf{Affine locking and translation inefficiency:} Existing methods universally default to freezing the bias terms~\cite{peft}.
    In physical systems, pure spatial convection (Galilean translation) is equivalent to an affine bias shift~\cite{bitfit}.
    Forcing high-dimensional multiplicative weight matrices to simulate simple additive translations is highly inefficient and unnecessarily consumes the core matrix's non-linear fitting capacity.
\end{itemize}

SciML urgently demands a novel fine-tuning paradigm that simultaneously obliterates ambient redundancy, subspace locking, and high-frequency spectral deficiency, while strictly adhering to first-principle PDE physics.
To shatter these theoretical deadlocks, we abandon the heuristic patching of adapters and propose a foundational micro-architecture exclusively tailored for physical operator extrapolation, as visually compared with existing paradigms in Fig.~\ref{fig.compare}: \textbf{M}anifold-\textbf{O}rthogonal \textbf{D}ual-spectrum \textbf{E}xtrapolation (MODE). 

In summary, the main contributions of this work are as follows:

\begin{itemize}
    \item \textbf{Proposing the MODE paradigm:} Bypassing the heuristic patching of adapters, MODE intrinsically resolves the Pareto disaster by decoupling physical evolution into principal-spectrum phase shifting, residual-spectrum awakening, and affine translation, establishing a new bedrock for operator adaptation.
    \item \textbf{Principal-spectrum dense mixing:} To break subspace locking without parameter bloat, MODE strictly confines geometric evolution within the absolutely frozen pre-trained orthogonal bases ($U_k, V_k^{\mathsf{T}}$).
    By embedding a microscopic, unconstrained dense core matrix ($C_{k \times k}$), it unlocks cross-modal energy transfer with a negligible $\mathcal{O}(k^2)$ parameter cost, enabling full-degree-of-freedom rotational phase shifts while completely rejecting the ambient space search.
    \item \textbf{Residual-spectrum awakening:} Addressing the severe loss of high frequencies, we pioneer a ``Dual-Spectrum'' mechanism.
    By creatively repurposing the frozen residual high-frequency matrix as a ``dormant dictionary,'' MODE introduces merely a \textit{single trainable scalar ($\tau$)} to dynamically ignite hidden shockwaves and broadband modes.
    It is the first PEFT mechanism capable of breaching the $\sim 10^0$ error barrier towards the full fine-tuning limit without expanding high-dimensional matrices.
    \item \textbf{Affine Galilean unlocking:} Addressing affine locking, we explicitly isolate pure spatial Galilean translation from the complex non-linear weight evolution by unlocking a lightweight, independent affine bias drift ($\Delta b$).
    This decouples kinematic phase shifts from dynamical stiff reactions, significantly easing the optimization burden.
    \item \textbf{Establishing absolute Pareto dominance:} We theoretically and empirically demonstrate that MODE effectively compresses the spatial complexity to $\mathcal{O}(k^2 + d_{out} + 1)$, perfectly matching the absolute minimal memory footprint of native SVD.
    Simultaneously, empowered by dual-spectrum extrapolation, its generalization error significantly surpasses the limitations of all existing PEFTs, establishing absolute Pareto dominance in parameter-efficient scientific computing.
\end{itemize}

\section{Related Works}
\label{sec:related_works}

Our work resides at the critical intersection of Scientific Machine Learning (SciML) and Parameter-Efficient Fine-Tuning (PEFT).
In this section, we review the foundational developments in both domains and explicitly diagnose the theoretical gaps that occur when migrating PEFT methods to physics-informed machine learning.

\subsection{Scientific Machine Learning and Parameterized PINNs}
Traditional numerical methods such as finite element methods and finite difference methods have clear pros and cons~\cite{patidar2016nonstandard, li1997adaptive, srirekha2010infinite}.
The more accurate the results, the more expensive the calculation of numerically approximated formulas.
To alleviate these computational bottlenecks, researchers have increasingly turned to machine learning approaches~\cite{karniadakis2021physics, cuomo2022scientific}.
Following early trials using the Galerkin or Deep Ritz methods~\cite{rudd2015constrained, yu2018deep}, Physics-Informed Neural Networks (PINNs)~\cite{raissi2019physics} proposed a transformative way of using deep learning to solve general governing PDEs in a physically sound and easy-to-formulate computational formalism.
Other lines of research analyze operator learning for differential equations~\cite{li2020fourier, gupta2021multiwavelet}, but PINNs maintain their potential by focusing on governing equations which describe physical phenomena.
There have been various strategies to impose physical constraints on neural networks~\cite{LagrangianNN, rudd2015constrained, lee2021machine}.
Most focus on imposing constraints on outputs or injecting specific physical conditions into networks.
As a simple but effective solution, PINNs directly impose physical conditions by utilizing the governing equation itself as a loss term ($\mathcal{L}_f$)~\cite{raissi2019physics}.
Initial and boundary conditions are similarly enforced as data matching losses ($\mathcal{L}_u, \mathcal{L}_b$).
Despite their elegance, PINNs possess recognized weaknesses: they struggle to learn certain classes of PDEs exhibiting high oscillation or sharp transitions~\cite{krishnapriyan2021characterizing}, and gradient-based training often converges to sub-optimal local minima.
Consequently, PINNs have evolved to resolve issues inherent to vanilla architectures.
Architectural enhancements include low-rank extensions and hypernetworks for model efficiency~\cite{cho2024hypernetwork}, and separable designs for efficient training~\cite{cho2024separable}.
Systematic assessments and new sampling strategies have been investigated in PINNACLE~\cite{lau2023pinnacle}.
Furthermore, researchers have combined PINNs with symbolic regression~\cite{podina2023universal} and devised operator preconditioners~\cite{de2023operator}.
Novel optimizers like MultiAdam~\cite{yao2023multiadam} and energy natural gradient descent~\cite{muller2023achieving} have also been proposed.
To avoid repetitive training from scratch, Parameterized PINNs (P$^2$INNs) encode parameters into a latent representation to learn equation families.
Yet, effectively adapting pre-trained P$^2$INNs foundation models to extreme out-of-distribution (OOD) physical parameters via full fine-tuning remains computationally exorbitant, necessitating parameter-efficient adaptation mechanisms.

\subsection{Parameter-Efficient Fine-Tuning (PEFT)}
Recent advancements in large models~\cite{gpt2, T5, llama2, llama3, gemma} have emphasized the development of PEFT techniques to enhance adaptability and efficiency, avoiding the prohibitive cost of full-parameter fine-tuning~\cite{peft, adapter-tuning}.
A notable contribution in this field is Low-Rank Adaptation (LoRA)~\cite{lora}, which freezes the weights of pre-trained models and integrates trainable low-rank matrices into each layer.
We highlight recent approaches that further improve this architecture.
Vector-based Random Matrix Adaptation (VeRA)~\cite{vera} minimizes trainable parameters utilizing shared frozen low-rank random matrices and learning compact scaling vectors.
An alternative approach, Weight-Decomposed Low-Rank Adaptation (DoRA)~\cite{dora}, decomposes pre-trained weight matrices into magnitude and direction components to enhance learning capacity and training stability.
AdaLoRA~\cite{adalora} adaptively distributes the parameter budget across weight matrices based on their importance scores.
PiSSA (Principal Singular Values and Singular Vectors Adaptation)~\cite{pissa} is a variant where the low-rank matrices are initialized with the principal SVD components.
Lastly, FLoRA~\cite{flora} enhances LoRA by enabling each example in a mini-batch to utilize distinct low-rank weights, facilitating efficient batching.
Orthogonal Fine-tuning (OFT)~\cite{oft} modifies pre-trained weight matrices through orthogonal reparameterization to preserve essential information.
Butterfly Orthogonal Fine-tuning (BOFT)~\cite{boft} extends OFT's methodology by incorporating Butterfly factorization, applying multiplicative orthogonal weight updates, which improves parameter efficiency and fine-tuning flexibility.
In the SciML context, early attempts like Singular Value Fine-tuning (e.g., SVD-PINNs~\cite{Gao_2022}) attempted knowledge transfer by strictly freezing singular vectors and updating only the diagonal singular values. 

\section{Methodology}
\label{sec:method}

In this section, we first formulate the problem of solving parameterized PDEs and define the architecture of foundational neural operators (\cref{sec:preliminary}).
Subsequently, we deconstruct the structural failures of the Singular Value Decomposition (SVD) fine-tuning paradigm employed in prior frameworks (e.g., P$^2$INNs).
This analysis strictly motivates the formulation of our Manifold-Orthogonal Dual-spectrum Extrapolation (MODE) algorithm (\cref{sec:motivation}).

\begin{figure}
    \centering
    \includegraphics[width=.9\textwidth]{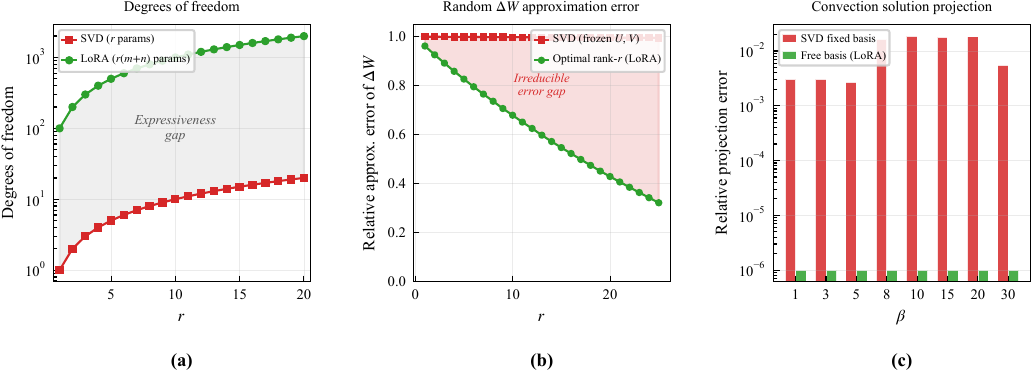}
    \vspace{-2mm}
    \caption{Structural limitations of SVD-based fine-tuning in P$^2$INNs. (a) Degrees of freedom comparison between SVD and LoRA. (b) Approximation error of SVD versus optimal rank-$r$ projection. (c) PDE scenario comparing fixed-basis and free-basis projections.}
\label{fig:svd_basis_limitation}
\end{figure}
\FloatBarrier

\subsection{Preliminary}
\label{sec:preliminary}

In scientific computing, many complex physical systems are governed by parameterized partial differential equations defined on a spatiotemporal domain $(x,t) \in \Omega \times [0, T]$.
As an illustrative example, we consider the parameterized Convection-Diffusion-Reaction (CDR) equation:
\begin{equation}
\frac{\partial u}{\partial t} + \beta\frac{\partial u}{\partial x} - \nu\frac{\partial^2 u}{\partial x^2} - \rho u (1 - u) = 0, \quad x \in \Omega, \; t \in [0,T],
\label{eq:parameterized_pde}
\end{equation}
where the continuous state variable $u(x,t; \pmb{\mu})$ models the physical field, governed by the parameter vector $\pmb{\mu} = [\beta, \nu, \rho]^\top$ representing the convective velocity, diffusion coefficient, and reaction rate, respectively.
In real-time simulation and multi-query scenarios, evaluating solutions across numerous points in the parameter space via repetitive PINN training from scratch is computationally prohibitive.
These scenarios include inverse design and uncertainty quantification.

To address this challenge, Parameterized PINNs (P$^2$INNs)~\cite{cho2024parameterized} approximate the entire solution family as a continuous neural mapping $u_{\Theta}(x,t; \pmb{\mu})$.
Specifically, P$^2$INNs employ a modularized architecture comprising two independent encoders and a manifold decoder.
The encoders map the spatiotemporal coordinates and PDE parameters into latent representations, i.e., $\pmb{h}_{\text{coord}} = g_{\theta_c}(x,t)$ and $\pmb{h}_{\text{param}} = g_{\theta_p}(\pmb{\mu})$.
These latent vectors are then concatenated and processed by the manifold network $g_{\theta_g}$ to predict the physical field:
\begin{equation}
\hat{u}(x,t; \pmb{\mu}) = g_{\theta_g}\bigl([\pmb{h}_{\text{coord}};\; \pmb{h}_{\text{param}}]\bigr).
\label{eq:ppinn_forward}
\end{equation}

Once the foundational P$^2$INN is pre-trained on a source parameter distribution, its weight matrices encode a generalized physical manifold.
To rapidly adapt the pre-trained model to an out-of-distribution (OOD) physical parameter $\pmb{\mu}^*$, Singular Value Decomposition (SVD) modulation is typically employed to fine-tune the intermediate linear layers $\mathbf{W}_0 \in \mathbb{R}^{d_{\text{out}} \times d_{\text{in}}}$ within the manifold network $g_{\theta_g}$.

\begin{figure}
    \centering
    \includegraphics[width=.9\textwidth]{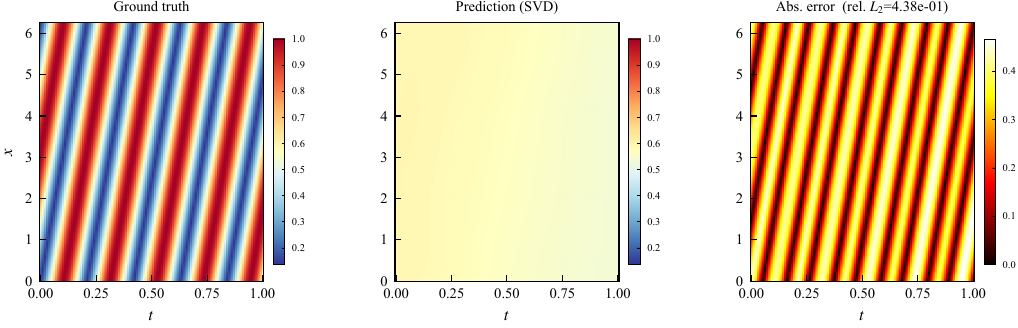}
    \vspace{-2mm}
    \caption{Subspace rotation failure of P$^2$INN-SVD when generalizing to the OOD convection equation $u_t + 30\,u_x = 0$.}
\label{fig:svd_subspace_rotation}
\end{figure}
\FloatBarrier

\subsection{Motivation}
\label{sec:motivation}

To perform rapid adaptation without catastrophic forgetting, the native P$^2$INN-SVD fine-tuning strategy decomposes the pre-trained weight matrix $\mathbf{W}_0$ via Singular Value Decomposition: $\mathbf{W}_0 = \mathbf{U} \mathbf{\Sigma} \mathbf{V}^\top$.
By truncating to a target rank $k \ll \min(d_{\text{in}}, d_{\text{out}})$, the weight is partitioned into a principal spectrum and a residual tail:
\begin{equation}
\mathbf{W}_0 = \underbrace{\mathbf{U}_k \mathbf{\Sigma}_k \mathbf{V}_k^\top}_{\text{Principal Spectrum } \mathbf{W}_{\text{prin}}} + \underbrace{\mathbf{W}_{\text{res}}}_{\text{Residual Spectrum}},
\label{eq:svd_partition}
\end{equation}
where $\mathbf{U}_k \in \mathbb{R}^{d_{\text{out}} \times k}$ and $\mathbf{V}_k \in \mathbb{R}^{d_{\text{in}} \times k}$ represent the foundational orthogonal physical bases.
The P$^2$INN-SVD paradigm strictly freezes these singular vectors and exclusively learns a new diagonal scaling matrix $\mathbf{\Lambda}_k = \text{diag}(\pmb{\alpha}) \in \mathbb{R}^{k \times k}$.
Simultaneously, it implicitly discards the residual matrix $\mathbf{W}_{\text{res}}$ and freezes the original bias $\mathbf{b}_0$.
The adapted layer is thus defined as:
\begin{equation}
\mathbf{W}_{\text{SVD}} = \mathbf{U}_k \mathbf{\Lambda}_k \mathbf{V}_k^\top.
\label{eq:svd_update}
\end{equation}

While this mechanism requires minimal trainable parameters, its structural limitations become evident when examined through the lens of PDE functional mappings.
As shown in \cref{fig:svd_basis_limitation}, the frozen singular vectors and hard spectrum truncation jointly restrict the model's adaptability under OOD parameter generalization.
Specifically, the SVD decomposition in \cref{eq:svd_partition} partitions the pre-trained weight into a principal spectrum $\mathbf{W}_{\text{prin}} = \mathbf{U}_k \mathbf{\Sigma}_k \mathbf{V}_k^\top$ and a discarded residual $\mathbf{W}_{\text{res}}$.
The subsequent fine-tuning in \cref{eq:svd_update} freezes the orthogonal bases $\mathbf{U}_k$, $\mathbf{V}_k$ and only learns a diagonal rescaling $\mathbf{\Lambda}_k$.
This confines the adapted weight to a rigid, axis-aligned subspace of the original manifold.

To rigorously diagnose these limitations, we first characterize the fundamental deficiencies of the P$^2$INN-SVD paradigm from four complementary theoretical perspectives.
These perspectives span function space expressivity, operator structure compatibility, and optimization landscape geometry.
Subsequently, we distill these theoretical insights into three concrete structural deadlocks, which pinpoint exactly where the failure manifests in the network's forward computation.

\textit{(i) Expressivity Collapse from Basis Freezing.}
By freezing $\mathbf{U}_k$ and $\mathbf{V}_k$, the hypothesis space of the adapted layer is strictly confined to a low-rank linear subspace:
\begin{equation}
\mathcal{H}_{\text{SVD}} = \{\mathbf{U}_k \mathbf{\Lambda}_k \mathbf{V}_k^\top \mid \alpha_i \in \mathbb{R}\}.
\label{eq:hypothesis_space}
\end{equation}
As illustrated in \cref{fig:svd_basis_limitation}(a), this freezing mechanism drastically restricts the trainable degrees of freedom.
When complex fluid systems undergo nonlinear topological phase transitions across critical parameter thresholds (e.g., laminar-to-turbulent transition), emergent high-frequency physical modes $\mathbf{e}_{\text{new}}$ arise in the target field.
Since these modes are absent from the source-task low-frequency bases (i.e., $\mathbf{e}_{\text{new}} \notin \text{span}(\mathbf{U}_k)$ and $\mathbf{e}_{\text{new}} \notin \text{span}(\mathbf{V}_k)$), the optimal projection error within this restricted subspace admits a strictly positive lower bound $\epsilon$:
\begin{equation}
\inf_{\mathbf{W} \in \mathcal{H}_{\text{SVD}}} \|u_{\text{true}} - \mathcal{N}_{\mathbf{W}}(\mathbf{x})\| \ge \epsilon > 0.
\label{eq:expressivity_bound}
\end{equation}
\Cref{eq:hypothesis_space,eq:expressivity_bound} establish an \textit{irreducible approximation barrier}: no diagonal rescaling within the frozen subspace can close the gap to the true solution manifold.

\textit{(ii) The ``Pure Linear Modulation'' Paradox in Nonlinear Manifold Mappings.}
The capacity of neural networks to approximate highly complex PDE operators relies on the nonlinear activation mappings within their hierarchical structure, i.e., $\mathbf{h}^{(l+1)} = \sigma(\mathbf{W}^{(l)}\mathbf{h}^{(l)} + \mathbf{b}^{(l)})$.
SVD modulation restricts the parameter update to $\Delta\mathbf{W} = \mathbf{U}_k(\mathbf{\Lambda}_k - \mathbf{\Sigma}_k)\mathbf{V}_k^\top$, which constitutes a purely linear affine diagonal transformation in the parameter space.
However, variations in PDE control parameters (e.g., Reynolds number) induce drastic nonlinear diffeomorphisms $\Phi: \mathcal{M}_{\text{source}} \to \mathcal{M}_{\text{target}}$ in the feature space.
Topologically, a few orthogonal-direction linear diagonal scalings $\mathbf{\Lambda}_k$ are fundamentally incapable of compensating for the nonlinear distortion and folding of the feature manifold.
This leads to unbounded truncation error:
\begin{equation}
\|\Phi(\mathbf{h}) - \sigma((\mathbf{W} + \Delta\mathbf{W})\mathbf{h} + \mathbf{b})\| \to \infty.
\label{eq:nonlinear_paradox}
\end{equation}
As shown in \cref{fig:svd_basis_limitation}(b), the divergence in \cref{eq:nonlinear_paradox} results in a severe approximation error gap between the SVD-based linear modulation and the optimal rank-$r$ projection.

\textit{(iii) Fundamental Applicability Blind Spot in Non-Self-Adjoint Dissipative Systems.}
For hypersonic flows or thermal conduction systems with strong convective dissipation, the spatial evolution operator $\mathcal{L}$ (e.g., the spatial terms in $\frac{\partial u}{\partial t} + \beta\nabla u = \nu\Delta u$) is non-self-adjoint (i.e., $\mathcal{L} \neq \mathcal{L}^*$).
This asymmetry dictates that the true physical eigenmode system $\{\phi_k\}$ must be non-orthogonal (i.e., $\langle \phi_i, \phi_j \rangle \neq 0$ for $i \neq j$).
However, the standard SVD fundamentally enforces absolute orthogonality constraints on the singular vector bases in Euclidean space ($\mathbf{U}_k^\top\mathbf{U}_k = \mathbf{I}$, $\mathbf{V}_k^\top\mathbf{V}_k = \mathbf{I}$).
As depicted in \cref{fig:svd_basis_limitation}(c), fixed-basis projections fail catastrophically in such PDE scenarios, whereas free-basis approaches can flexibly adapt to the non-orthogonal structures required by dissipative systems.
Consequently, imposing symmetric orthogonal SVD bases to represent asymmetric dissipative dynamics rigidly severs the energy transfer pathways sustaining convective directionality.
Furthermore, it causes the variational integral objective to diverge unboundedly under infinite-dimensional mappings.

These theoretical deficiencies in function space, operator structure, and optimization landscape manifest as three concrete \textit{structural deadlocks} in the network's forward computation.
Specifically, expressivity collapse and linear modulation paradox cause rigid diagonal coupling (Deadlock~I).
Orthogonality-dissipation mismatch prevents safe truncation of the residual spectrum (Deadlock~II).
Finally, optimization collapse is exacerbated by forcing multiplicative parameters to absorb additive coordinate shifts (Deadlock~III).

\textbf{Deadlock I: Subspace locking via diagonal rigidity.}
Expanding the forward pass $\mathbf{W}_{\text{SVD}} \mathbf{x}$ exposes a strict geometric limitation:
\begin{equation}
\mathbf{W}_{\text{SVD}} \mathbf{x} = \sum_{i=1}^k \alpha_i \, \mathbf{u}_i \bigl(\mathbf{v}_i^\top \mathbf{x}\bigr),
\label{eq:subspace_locking}
\end{equation}
where $\mathbf{u}_i$ and $\mathbf{v}_i$ denote the $i$-th columns of $\mathbf{U}_k$ and $\mathbf{V}_k$, respectively.
\Cref{eq:subspace_locking} dictates that any input feature projection along $\mathbf{v}_i$ \textit{must strictly} map to its corresponding homologous output basis $\mathbf{u}_i$.
The diagonal structure of $\mathbf{\Lambda}_k$ rigorously prohibits off-diagonal cross-modal energy transfer.
Specifically, it prevents generating associative terms of the form $\mathbf{u}_i (\mathbf{v}_j^\top \mathbf{x})$ for $i \neq j$.
In convection-dominated regimes (e.g., large shifts in $\beta$), tracking physical wave propagation intrinsically requires continuous spatial rotation and phase shifting of feature representations.
By restricting the update to pure scalar amplitudes $\alpha_i$, P$^2$INN-SVD imposes an absolute ``Subspace Lock''.
This paralyzes the network's capacity to rotate the physical feature manifold.
As illustrated in \cref{fig:svd_subspace_rotation}, when fine-tuning the pre-trained model for out-of-distribution (OOD) parameter generalization, SVD-based adaptation is confined to diagonal singular value scaling (amplitude modulation only).
Consequently, it is incapable of performing the cross-modal subspace rotation required to track physical feature manifold changes under OOD convection-dominated PDE mutations.

\textbf{Deadlock II: The single-spectrum truncation trap.}
Existing SVD-based fine-tuning algorithms universally treat the truncated residual $\mathbf{W}_{\text{res}}$ as negligible noise.
However, according to the Kolmogorov $n$-width theory~\cite{pinkus1985n}, in dynamical physical systems these tail singular values densely encode critical \textit{high-frequency physical modes}.
These modes include steep boundary layers, sharp gradients, and shockwave structures.
When adapting to stiff reaction rates (e.g., large $\rho$), the network demands a surge in high-frequency expressivity.
By applying a hard truncation ($\mathbf{W}_{\text{adapted}} \approx \mathbf{W}_{\text{prin}}$), P$^2$INN-SVD inadvertently acts as a rigid low-pass filter.
This irreversibly collapses the model's high-frequency topological capacity.
This explains why the method encounters a strict generalization error barrier well above the full fine-tuning limit.

\textbf{Deadlock III: Affine locking in Galilean translations.}
In the geometric representation of PDEs, a pure spatial advection (Galilean translation) is isomorphic to an affine bias drift in the neural network coordinate system.
P$^2$INN-SVD universally defaults to freezing the pre-trained bias vector $\mathbf{b}_0$.
Forcing the high-dimensional multiplicative weight matrices ($\mathbf{U}$, $\mathbf{V}$) to simulate simple additive coordinate translations induces severe parameter redundancy.
Furthermore, it leads to gradient entanglement and optimization rigidity.

Driven by the imperative to fundamentally resolve these topological deadlocks, we propose the MODE algorithm.
This algorithm explicitly unblocks subspace locking via cross-modal phase shifting.
Additionally, it re-awakens the discarded high-frequency spectrum and decouples affine translations from the multiplicative weight space.

\textbf{Remark: Why existing PEFT methods cannot resolve these deadlocks.}
A natural question arises: can mainstream Parameter-Efficient Fine-Tuning (PEFT) methods be directly ported to remedy the above failures?
We argue that the answer is negative.
Each family of methods inherits its own structural incompatibility with physics-informed machine learning.
Additive low-rank adapters such as LoRA~\cite{lora} inject an unconstrained perturbation $\Delta\mathbf{W} = \mathbf{B}\mathbf{A}$ (where $\mathbf{B} \in \mathbb{R}^{d_{\text{out}} \times r}$, $\mathbf{A} \in \mathbb{R}^{r \times d_{\text{in}}}$, $r \ll \min(d_{\text{in}}, d_{\text{out}})$) into the pre-trained weight:
\begin{equation}
\mathbf{h} = \mathbf{W}_0 \mathbf{x} + \mathbf{B}\mathbf{A}\mathbf{x}.
\label{eq:lora_update}
\end{equation}
While the additive update in \cref{eq:lora_update} lifts the diagonal rigidity of Deadlock~I, the heuristic search in the unconstrained ambient space $\mathbb{R}^{d_{\text{out}} \times d_{\text{in}}}$ recklessly discards the orthogonal physical priors crystallized during pre-training.
This triggers severe gradient entanglement when fitting stiff physical derivatives.
Variants such as VeRA~\cite{vera} ($\mathbf{h} = \mathbf{W}_0 \mathbf{x} + \mathbf{\Lambda}_b\mathbf{B}\mathbf{\Lambda}_d\mathbf{A} \mathbf{x}$) and DoRA~\cite{dora} ($\mathbf{h} = \vm\frac{\mathbf{W}_0+\mathbf{B}\mathbf{A}}{\|\mathbf{W}_0+\mathbf{B}\mathbf{A}\|_c}\mathbf{x}$) inherit the same ambient-space limitation.
Conversely, spectral and orthogonal methods such as OFT~\cite{oft} and BOFT~\cite{boft} apply multiplicative orthogonal updates $\mathbf{h} = (\mathbf{R}\cdot \mathbf{W}_0)\mathbf{x}$.
These updates preserve the pre-trained manifold structure but still universally truncate the residual tail spectrum, leaving Deadlock~II unresolved.
Furthermore, none of these methods explicitly decouple the affine bias component (Deadlock~III).
In summary, no existing PEFT method simultaneously addresses all three deadlocks.
This motivates the design of our MODE algorithm.

\subsection{MODE Formulation}
\label{sec:mode}

To fundamentally shatter the deadlocks of native SVD fine-tuning while strictly avoiding the parameter bloat associated with ambient space searches, we propose the Manifold-Orthogonal Dual-spectrum Extrapolation (MODE) method.
MODE explicitly rejects simple additive perturbations in high-dimensional Euclidean spaces.
Instead, it reconstructs the evolution space by defining a minimalist trainable parameter set:
\begin{equation}
    \Omega_{\text{MODE}} = \big\{ \mathbf{\Phi} \in \mathbb{R}^{k \times k}, \tau \in \mathbb{R}, \Delta \mathbf{b} \in \mathbb{R}^{d_{out}} \big\}.
\label{eq:mode_parameter_set}
\end{equation}
\Cref{eq:mode_parameter_set} is the compact search space from which all subsequent MODE operators are constructed.

Let $\mathbf{W}_0 \in \mathbb{R}^{d_{out} \times d_{in}}$ and $\mathbf{b}_0 \in \mathbb{R}^{d_{out}}$ denote the full-precision weight matrix and bias vector of the $l$-th target hidden layer in a P$^2$INN, obtained after the Phase 1 full pre-training converges.
By applying Singular Value Decomposition (SVD) to $\mathbf{W}_0$ and imposing a truncation rank $k \ll \min(d_{out}, d_{in})$, we orthogonally decompose the pre-trained weight into a Principal Spectrum Manifold ($\mathcal{P}_k$) and a Residual Spectrum ($\mathcal{R}_k$):
\begin{equation*}
    \mathbf{W}_0 = \mathbf{U} \mathbf{\Sigma} \mathbf{V}^	op \equiv \underbrace{\mathbf{U}_k \mathbf{\Sigma}_k \mathbf{V}_k^	op}_{\text{Principal } (\mathcal{P}_k)} + \underbrace{\mathbf{W}_{\text{res}}}_{\text{Residual } (\mathcal{R}_k)},
\end{equation*}
where $\mathbf{U}_k \in \mathbb{R}^{d_{out} \times k}$ and $\mathbf{V}_k \in \mathbb{R}^{d_{in} \times k}$ constitute the absolute orthogonal bases of the truncated principal subspace.
$\mathbf{\Sigma}_k \in \mathbb{R}^{k \times k}$ is the diagonal principal singular value matrix.
The truncated high-frequency tail residual matrix is defined as $\mathbf{W}_{\text{res}} = \mathbf{W}_0 - \mathbf{U}_k \mathbf{\Sigma}_k \mathbf{V}_k^	op$.

In native Singular Value Fine-tuning (i.e., P$^2$INN-SVD), to extremely compress parameters, the update equation is severely restricted to a diagonal scaling vector $\Delta \boldsymbol{\alpha} \in \mathbb{R}^k$.
This structure strictly prohibits cross-modal mappings $\mathbf{u}_i \mathbf{u}_j^	op \ (i \neq j)$:
\begin{align}
    \widetilde{\mathbf{W}}_{\text{SVD}} &= \mathbf{U}_k \big( \mathbf{\Sigma}_k + \text{diag}(\Delta \boldsymbol{\alpha}) \big) \mathbf{V}_k^	op, \label{eq:p2inn_svd} \\
    \widetilde{\mathbf{W}}_{\text{SVD}} \mathbf{x} &= \sum_{i=1}^k (\sigma_i + \Delta \alpha_i) \mathbf{u}_i (\mathbf{v}_i^	op \mathbf{x}). \notag
\end{align}

\begin{remark}[Pathology of Native SVF]
\label{rmk:svf}
Native Singular Value Fine-Tuning (e.g., P$^2$INN-SVD) forcefully imposes a diagonal constraint $\Omega_{\text{SVD}} = \{ \Delta \boldsymbol{\alpha} \in \mathbb{R}^k \}$. 
Algebraically, this strictly dictates the cross-term coefficients to zero, i.e., $\text{Coeff}(\mathbf{u}_i \mathbf{v}_j^	op) \equiv 0 \ (\forall i \neq j)$, thoroughly depriving the feature manifold of the Lie algebraic rotational degrees of freedom required to fit macroscopic convective phase-shifts.
Concurrently, it forcefully assumes $\mathbf{W}_{\text{res}} \to \mathbf{0}$, causing the topological capacity to approximate broadband high-frequency shockwave energy to permanently collapse.
\end{remark}

\begin{remark}[Pathology of Additive Low-Rank \& Spectral Variants]
\label{rmk:lora}
Projecting the update matrix of additive low-rank fine-tuning (e.g., LoRA) into the pre-trained singular basis coordinate system yields $\mathbf{\Sigma}_k + \mathbf{U}_k^	op (\mathbf{B} \mathbf{A}) \mathbf{V}_k$.
Although the dense off-diagonal terms implicitly unlock modal mixing, this paradigm completely abandons the orthogonal physical priors.
Its spatial complexity $\mathcal{O}(r (d_{in} + d_{out}))$ severely violates the minimalist parameter lower bound, plunging optimization into an ambient space blind search prone to condition number collapse.
Furthermore, modern spectral variants (e.g., SVFT) unlock cross-mappings within the $k$-dimensional subspace but remain trapped by the single-spectrum truncation assumption (discarding the dictionary $\mathbf{W}_{\text{res}}$), leaving their generalization pathways to physical stiff mutations completely severed.
\end{remark}

From the perspective of PDE functional mappings, Eq.~\eqref{eq:p2inn_svd} algebraically exposes three fatal deadlocks, forming the direct motivation for our reconstruction:
\begin{itemize}
    \item \textbf{Subspace locking:} Constraining the update to a diagonal matrix $\text{diag}(\Delta \boldsymbol{\alpha})$ dictates that the projection of the input vector on $\mathbf{v}_i$ \textit{must} map exclusively to its homologous $\mathbf{u}_i$.
    Cross-modal mapping terms $\mathbf{u}_i \mathbf{v}_j^	op$ ($i \neq j$) are strictly prohibited.
    This topologically paralyzes the rotational and shear degrees of freedom of the manifold, rendering macroscopic phase evolution (e.g., convective translations) entirely impossible.
    \item \textbf{Single-spectrum truncation trap:} Eq.~\eqref{eq:p2inn_svd} forcefully imposes $\mathbf{W}_{\text{res}} \to \mathbf{0}$.
    In physical systems, Kolmogorov $n$-width theory dictates that these tiny tail singular values densely encode critical high-frequency modes such as shockwaves and steep boundaries.
    Truncating the residual is equivalent to applying a hard low-pass filter.
    This permanently deprives the model of the high-frequency topological capacity required to approach the Phase 1 theoretical limit.
    \item \textbf{Affine locking:} Defaulting to a frozen bias $\mathbf{b}_0$ forces the high-dimensional multiplicative weight matrices to simulate Galilean translations in pure spatial dimensions.
    Since such translations are inherently additive affine shifts, encoding them through multiplicative weights leads to extreme parameter inefficiency.
\end{itemize}

To obliterate these deadlocks, MODE decouples the evolution process into three independent components:

\textbf{(i) Principal-spectrum dense mixing:} By introducing an unconstrained microscopic dense matrix $\mathbf{\Phi}$, its off-diagonal elements $\phi_{ij}$ directly bridge the cross-modal energy mappings between orthogonal physical modes, unlocking full-degree-of-freedom Lie algebraic rotations:
\begin{align}
    \widetilde{\mathbf{W}}_{\text{prin}} &= \mathbf{U}_k (\mathbf{\Sigma}_k + \mathbf{\Phi}_k) \mathbf{V}_k^	op, \\
    \widetilde{\mathbf{W}}_{\text{prin}} \mathbf{x} &= \sum_{i=1}^k \sum_{j=1}^k (\sigma_i \delta_{ij} + \phi_{ij}) \mathbf{u}_i (\mathbf{v}_j^	op \mathbf{x}),
\label{eq:mode_principal_mixing}
\end{align}
where $\delta_{ij}$ denotes the Kronecker delta.

\textbf{(ii) Residual-spectrum awakening:} By introducing a unique, globally trainable scalar $\tau$, we dynamically activate the frozen residual dictionary in an isotropic, proportional, and full-volume manner:
\begin{equation}
    \widetilde{\mathbf{W}}_{\text{res}} = \tau \cdot \mathbf{W}_{\text{res}}.
\label{eq:mode_residual_awakening}
\end{equation}

\textbf{(iii) Affine Galilean unlocking:} By introducing an independent drift vector $\Delta \mathbf{b}$, pure spatial translations are explicitly decoupled from the weight matrix:
\begin{equation}
    \widetilde{\mathbf{b}} = \mathbf{b}_0 + \Delta \mathbf{b}.
\label{eq:mode_affine_unlocking}
\end{equation}

Strictly fusing the three operators in \cref{eq:mode_principal_mixing,eq:mode_residual_awakening,eq:mode_affine_unlocking} yields the standard forward evolution polynomial of MODE:
\begin{equation}
    \mathbf{y}_{\text{MODE}} = \Big[ \mathbf{U}_k (\mathbf{\Sigma}_k + \mathbf{\Phi}) \mathbf{V}_k^	op + \tau \mathbf{W}_{\text{res}} \Big] \mathbf{x} + \Big( \mathbf{b}_0 + \Delta \mathbf{b} \Big).
\label{eq:mode_standard_forward}
\end{equation}
The standard form in \cref{eq:mode_standard_forward} exposes the three MODE pathways before algebraic memory reduction.

To eliminate the prohibitive $\mathcal{O}(d_{out} \times d_{in})$ physical memory disaster caused by explicitly instantiating the massive high-frequency dictionary $\mathbf{W}_{\text{res}}$, we substitute the analytical identity to execute a residual-free algebraic reconstruction:
\begin{align}
    \widetilde{\mathbf{W}}_{\text{MODE}} &= \mathbf{U}_k (\mathbf{\Sigma}_k + \mathbf{\Phi}_k) \mathbf{V}_k^	op + \tau \big( \mathbf{W}_0 - \mathbf{U}_k \mathbf{\Sigma}_k \mathbf{V}_k^	op \big) \notag \\
    &= \tau \mathbf{W}_0 + \mathbf{U}_k \mathbf{\Sigma}_k \mathbf{V}_k^	op - \tau \mathbf{U}_k \mathbf{\Sigma}_k \mathbf{V}_k^	op + \mathbf{U}_k \mathbf{\Phi}_k \mathbf{V}_k^	op \notag \\
    &= \tau \mathbf{W}_0 + \mathbf{U}_k \Big( \mathbf{\Phi}_k + (1 - \tau)\mathbf{\Sigma}_k \Big) \mathbf{V}_k^	op.
\label{eq:mode_weight_reconstruction}
\end{align}

Extracting the orthogonal bases in \cref{eq:mode_weight_reconstruction} as common factors derives the ultimate highly efficient computational graph that fully reuses the native foundational weights and thoroughly eradicates the memory footprint of $\mathbf{W}_{\text{res}}$:
\begin{equation}
    \mathbf{y}_{\text{MODE}} = \tau (\mathbf{W}_0 \mathbf{x}) + \mathbf{U}_k \Big[ \mathbf{\Phi}_k + (1 - \tau)\mathbf{\Sigma}_k \Big] (\mathbf{V}_k^	op \mathbf{x}) + \big( \mathbf{b}_0 + \Delta \mathbf{b} \big).
    \label{eq:mode_efficient_forward}
\end{equation}

To circumvent early-stage destruction of the physical manifold, we provide the following theoretical guarantees:

\begin{theorem}[Zero-Degradation Exact Recovery]
\label{thm:recovery}
To circumvent early-stage physical manifold destruction caused by random noise initialization inherent in additive PEFT paradigms, MODE is constrained by strict mathematical identity boundary conditions:
\begin{equation}
    \mathbf{\Phi}_k^{(t=0)} = \mathbf{0}_{k \times k}, \quad \tau^{(t=0)} = 1.0, \quad \Delta \mathbf{b}^{(t=0)} = \mathbf{0}_{d_{out}}.
\label{eq:mode_initialization}
\end{equation}
\end{theorem}
\begin{proof}
Substituting the identity conditions in \cref{eq:mode_initialization} into the efficient forward computational graph in \cref{eq:mode_efficient_forward}, at the initial optimization step $t=0$:
\begin{align}
    \mathbf{y}^{(0)}_{\text{MODE}} &= 1.0 \cdot (\mathbf{W}_0 \mathbf{x}) \notag \\
    &\quad + \mathbf{U}_k \Big[ \mathbf{0}_{k \times k} + (1 - 1.0)\mathbf{\Sigma}_k \Big] (\mathbf{V}_k^	op \mathbf{x}) \notag \\
    &\quad + \big( \mathbf{b}_0 + \mathbf{0}_{d_{out}} \big) \notag \\
    &\equiv \mathbf{W}_0 \mathbf{x} + \mathbf{b}_0.
\label{eq:mode_initial_forward}
\end{align}
The infinitesimal terms strictly zero out.
\Cref{eq:mode_initial_forward} mathematically guarantees a $100\%$ lossless and seamless takeover of the pre-trained foundational physical operator at the fine-tuning starting point.
\end{proof}

\begin{theorem}[Asymptotic Pareto Spatial Complexity]
\label{thm:complexity}
The exact spatial complexity of the newly introduced parameters for a single-layer forward pass in MODE is determined by the algebraic sum of its micro-components:
\begin{equation}
    \begin{aligned}
    |\Omega_{\text{MODE}}| &= \dim(\mathbf{\Phi}) + \dim(\tau) + \dim(\Delta \mathbf{b}) \\
    &= k^2 + 1 + d_{out}.
    \end{aligned}
\label{eq:mode_complexity}
\end{equation}
\end{theorem}
\begin{proof}
Based on \cref{eq:mode_complexity} and the Kolmogorov $n$-width theory, the continuous physical manifold governed by PDEs inevitably dictates an extremely narrow truncation rank ($k \ll \min(d_{out}, d_{in})$).
Consequently, $k^2$ asymptotically tends to a minimal constant, and its asymptotic complexity limit perfectly collapses to the theoretical absolute lower bound equivalent to native minimalist SVD:
\begin{equation*}
    \mathcal{O}\big( |\Omega_{\text{MODE}}| \big) = \mathcal{O}(k^2 + d_{out}) \sim \mathcal{O}(d_{out}).
\end{equation*}

Thus, MODE theoretically achieves a Pareto-optimal balance between parameter efficiency and topological generalization capacity.
\end{proof}

\subsection{MODE Modulation}
\label{sec:Modulation}

Our framework operates in two distinct phases: (1) foundational pre-training on a source parameter distribution, and (2) parameter-efficient out-of-distribution (OOD) fine-tuning. We first briefly outline the foundational architecture and subsequently detail the proposed MODE fine-tuning architecture.

\paragraph{P$^2$INN Architecture}
To efficiently solve families of parameterized PDEs, we adopt the modularized Parameterized Physics-Informed Neural Network (P$^2$INN) as our foundational base model. It approximates the continuous physical field $u_{\Theta}(x,t; \pmb{\mu})$ via three sub-networks: a spatiotemporal coordinate encoder $g_{\theta_c}$, a PDE parameter encoder $g_{\theta_p}$, and a manifold decoder network $g_{\theta_g}$.

Specifically, the encoders extract latent representations $\pmb{h}_{\text{coord}}$ and $\pmb{h}_{\text{param}}$ from their respective inputs using stacked fully-connected (FC) layers equipped with non-linear activations:
\begin{equation}
    \pmb{h}_{\text{coord}} = g_{\theta_c}(x,t), \quad \pmb{h}_{\text{param}} = g_{\theta_p}(\pmb{\mu}).
\label{eq:latent_embeddings}
\end{equation}
These decoupled representations are then concatenated ($\oplus$) and processed by the manifold network to infer the continuous PDE solution:
\begin{equation}
    \hat{u}(x,t; \pmb{\mu}) = g_{\theta_g}(\pmb{h}_{\text{coord}} \oplus \pmb{h}_{\text{param}}).
\label{eq:model_forward}
\end{equation}
By explicitly encoding the PDE parameters into a high-dimensional latent space rather than treating them as mere input coordinates, the foundational model intrinsically constructs a generalized physical manifold that captures the shared dynamics across various equations. 
The latent encodings in \cref{eq:latent_embeddings} and the manifold prediction in \cref{eq:model_forward} define the base P$^2$INN mapping before MODE adaptation.
During the pre-training phase, the global parameter set $\Theta = \{ \theta_c, \theta_p, \theta_g \}$ is jointly optimized over a broad source distribution of PDEs. Once converged, the weights of $g_{\theta_g}$ crystallize the fundamental geometric priors of the PDE family and are \textit{strictly frozen}.

\begin{figure}
\centering
\includegraphics[width=.9\textwidth]{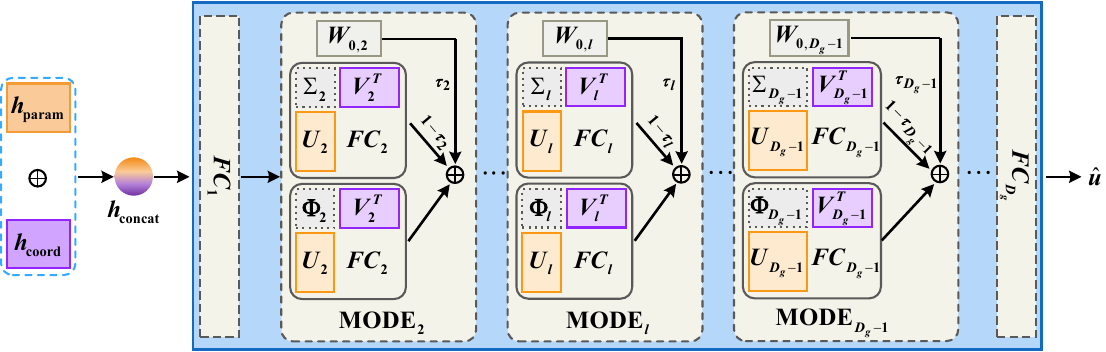}
\caption{\textbf{P$^2$INNs with MODE modulation.} Each intermediate layer of the manifold decoder is replaced by a dual-path MODE-adapted block: (i) the \textit{Frozen Full-Rank Path} scales the frozen weight output $\mathbf{W}_{0,l}\mathbf{h}^{(l-1)}$ by the trainable scalar $\tau_l$; (ii) the \textit{Ultra-Low-Rank Orthogonal Path} mixes cross-modal energy via the dense core $\mathbf{\Phi}_l$ within the compact $k$-dimensional subspace. Only $\{\mathbf{\Phi}_l, \tau_l, \Delta\mathbf{b}_l\}$ are trainable; all other components are strictly frozen.}
\label{fig:mode_arch}
\end{figure}
\FloatBarrier

\paragraph{MODE Modulation}
While the pre-trained foundational model encapsulates generalized physical priors, adapting it to extreme OOD parameters (e.g., highly stiff reactions or turbulent convections) demands profound topological transformations within the manifold network $g_{\theta_g}$. Native approaches attempt to solve this by replacing the intermediate fully-connected layers of $g_{\theta_g}$ with a rigid SVD modulation ($\widetilde{W}_{l} = \mathbf{U}_l \text{diag}(\boldsymbol{\alpha}_l) \mathbf{V}_l^\top$). However, as axiomatically diagnosed in Section~\ref{sec:mode}, this triggers severe subspace locking and high-frequency truncation, rendering the model mathematically incapable of generalizing to abrupt physical mutations.

To circumvent these structural defects, we seamlessly inject the proposed MODE modulation into the intermediate layers of the manifold network $g_{\theta_g}$ (typically $l = 2, 3, \dots, D_g-1$). 
For the $l$-th frozen foundational linear layer defined by pre-trained weight $\mathbf{W}_{0}^{(l)} \in \mathbb{R}^{d_{out} \times d_{in}}$ and bias $\mathbf{b}_{0}^{(l)} \in \mathbb{R}^{d_{out}}$, its orthogonal spectral decomposition under a strict truncation rank $k \ll \min(d_{out}, d_{in})$ is rigorously defined as $\mathbf{W}_{0}^{(l)} = \mathbf{U}_k^{(l)} \mathbf{\Sigma}_k^{(l)} {\mathbf{V}_k^{(l)}}^\top + \mathbf{W}_{\text{res}}^{(l)}$.

Instead of utilizing a rigid diagonal matrix, the MODE modulation introduces a minimalist trainable parameter set $\Omega_{\text{MODE}}^{(l)} = \big\{ \mathbf{\Phi}_k^{(l)} \in \mathbb{R}^{k \times k}, \tau^{(l)} \in \mathbb{R}, \Delta \mathbf{b}^{(l)} \in \mathbb{R}^{d_{out}} \big\}$, structurally remodeling the forward pass of the $l$-th neural layer. Given the input activation $\mathbf{h}^{(l-1)}$ from the previous layer, the theoretical forward architecture of a MODE-adapted layer is formulated as:
\begin{equation}
    \begin{aligned}
    \mathbf{h}^{(l)} = \sigma \Bigg( 
        &\Big[ \mathbf{U}_k^{(l)} (\mathbf{\Sigma}_k^{(l)} + \mathbf{\Phi}_k^{(l)}) {\mathbf{V}_k^{(l)}}^\top + \tau^{(l)} \mathbf{W}_{\text{res}}^{(l)} \Big] \mathbf{h}^{(l-1)} \\
        &+ \big(\mathbf{b}_0^{(l)} + \Delta \mathbf{b}^{(l)}\big) 
    \Bigg),
    \end{aligned}
\label{eq:mode_layer_standard}
\end{equation}
where $\sigma(\cdot)$ is the non-linear activation function (e.g., Tanh or SiLU).

Explicitly instantiating the massive high-frequency dictionary $\mathbf{W}_{\text{res}}^{(l)}$ in the computational graph incurs a prohibitive dense $\mathcal{O}(d_{out} \times d_{in})$ memory overhead, which severely contradicts the parameter-efficient philosophy. 
By substituting the analytical identity $\mathbf{W}_{\text{res}}^{(l)} \equiv \mathbf{W}_0^{(l)} - \mathbf{U}_k^{(l)} \mathbf{\Sigma}_k^{(l)} {\mathbf{V}_k^{(l)}}^\top$ into \cref{eq:mode_layer_standard} and factorizing the orthogonal bases, we restructure the architectural graph into a residual-free format:
\begin{equation}
    \begin{aligned}
    \mathbf{h}^{(l)} = \sigma \Bigg( 
        &\underbrace{\tau^{(l)} \big(\mathbf{W}_0^{(l)} \mathbf{h}^{(l-1)}\big)}_{\text{Frozen Full-Rank Path}} \\ 
        + &\underbrace{\mathbf{U}_k^{(l)} \Big[ \mathbf{\Phi}_k^{(l)} + (1 - \tau^{(l)})\mathbf{\Sigma}_k^{(l)} \Big] \big({\mathbf{V}_k^{(l)}}^\top \mathbf{h}^{(l-1)}\big)}_{\text{Ultra-Low-Rank Orthogonal Path}} \\
        + &\underbrace{\mathbf{b}_0^{(l)} + \Delta \mathbf{b}^{(l)}}_{\text{Affine Path}} 
    \Bigg).
    \end{aligned}
    \label{eq:mode_layer_forward}
\end{equation}

As illustrated in Fig.~\ref{fig:mode_arch} and dictated by Eq.~\eqref{eq:mode_layer_forward}, the MODE modulation strictly routes the forward signal through two computationally cheap streams. 
The \textit{Frozen Full-Rank Path} leverages the native, un-decomposed pre-trained weight $\mathbf{W}_0^{(l)}$ (requiring no gradient tracking), and its output is dynamically scaled by the residual awakener scalar $\tau^{(l)}$.
Simultaneously, the \textit{Ultra-Low-Rank Orthogonal Path} projects the input into a compact $k$-dimensional physical subspace via ${\mathbf{V}_k^{(l)}}^\top$, geometrically mixes it via the dense core matrix $\mathbf{\Phi}_k^{(l)}$, and projects it back via $\mathbf{U}_k^{(l)}$.

During the OOD fine-tuning phase targeting the new parameter $\pmb{\mu}^*$, the encoders $g_{\theta_c}$ and $g_{\theta_p}$, as well as the foundational layer weights $\{\mathbf{W}_0^{(l)}, \mathbf{b}_0^{(l)}\}$, are universally frozen. 
Gradient descent exclusively updates the ultra-compact MODE parameter set $\Omega_{\text{MODE}}^{(l)}$ for each adapted layer. This modulation design mathematically guarantees that the network inherits the extreme parameter efficiency of native SVD while structurally accommodating the high-frequency topological phase shifts essential for physical extrapolation.

\subsection{Training and Fine-tuning}
\label{sec:train_tune}

Our framework operates in two distinct phases: foundational pre-training on a source parameter distribution, and rapid fine-tuning for out-of-distribution (OOD) parameter generalization.

\paragraph{Parameterized physics-informed loss}
For a parameterized PDE defined by a differential operator $\mathcal{N}_{\pmb{\mu}}$ over the spatiotemporal domain $\Omega \times [0,T]$ (e.g., \cref{eq:parameterized_pde}), we train the P$^2$INN to satisfy both governing dynamics and constraints via a physics-informed loss:
\begin{equation}
    \mathcal{L}(\Theta) = w_{\text{PDE}}\mathcal{L}_{\text{PDE}} + w_{\text{IC}}\mathcal{L}_{\text{IC}} + w_{\text{BC}}\mathcal{L}_{\text{BC}},
\label{eq:loss}
\end{equation}
where $\mathcal{L}_{\text{PDE}}$ enforces the PDE residual on collocation points, and $\mathcal{L}_{\text{IC}}, \mathcal{L}_{\text{BC}}$ enforce initial and boundary conditions, respectively. 
The coefficients $w_{\text{PDE}}, w_{\text{IC}}, w_{\text{BC}} \in \mathbb{R}^+$ are balancing hyperparameters.
Let $\hat{u}(x,t;\pmb{\mu})$ denote the network prediction in \cref{eq:ppinn_forward}. We define the residual
\begin{equation}
    r_{\Theta}(x,t;\pmb{\mu}) := \mathcal{N}_{\pmb{\mu}}\!\left[\hat{u}(x,t;\pmb{\mu})\right],
\label{eq:pde_residual}
\end{equation}
which is evaluated via automatic differentiation. 
The loss terms are instantiated as mean-squared errors:
\begin{align}
    \mathcal{L}_{\text{PDE}} &= \frac{1}{N_f}\sum_{i=1}^{N_f}\left\|r_{\Theta}(x_f^{(i)},t_f^{(i)};\pmb{\mu})\right\|_2^2,\\
    \mathcal{L}_{\text{IC}} &= \frac{1}{N_u}\sum_{i=1}^{N_u}\left\|\hat{u}(x_u^{(i)},0;\pmb{\mu})-u_0(x_u^{(i)};\pmb{\mu})\right\|_2^2,\\
    \mathcal{L}_{\text{BC}} &= \frac{1}{N_b}\sum_{i=1}^{N_b}\left\|\hat{u}(x_b^{(i)},t_b^{(i)};\pmb{\mu})-u_b(x_b^{(i)},t_b^{(i)};\pmb{\mu})\right\|_2^2.
\label{eq:loss_components}
\end{align}
\Cref{eq:loss,eq:pde_residual,eq:loss_components} define the objective optimized in both pre-training and OOD fine-tuning.

\paragraph{Phase 1: Pre-training}
We follow the two-stage pipeline described in \cref{sec:train_tune}.
In Phase 1, the global parameter set $\Theta=\{\theta_c,\theta_p,\theta_g\}$ is jointly optimized over a broad source distribution of PDE parameters $\pmb{\mu} \sim p_{\text{source}}(\pmb{\mu})$.
Unlike standard PINNs that optimize for a single equation, each iteration samples a parameter $\pmb{\mu}$ and corresponding collocation/constraint points, minimizing the expectation $\mathbb{E}_{\pmb{\mu} \sim p_{\text{source}}}\big[\mathcal{L}(\Theta;\pmb{\mu})\big]$.
After convergence, the manifold network weights in $g_{\theta_g}$ encode the generalized physical manifold and are frozen to serve as the foundational model.

\begin{algorithm}[tb]
   \caption{Pre-training Procedure}
   \label{alg:train}
   \definecolor{codeblue}{rgb}{0.25,0.5,0.5}
\begin{algorithmic}[1]
   \STATE {\bfseries Input:} Source PDE parameter distribution $p_{\text{source}}$, spatiotemporal domain $\Omega \times [0, T]$.
   \STATE {\bfseries Initialize:} Global parameters $\Theta = \{ \theta_c, \theta_p, \theta_g \}$ of the P$^2$INN architecture.
   \STATE {\bfseries Hyper-parameters:} Learning rate $\eta$, batch size $B$, loss weights $w_{\text{IC}}, w_{\text{BC}}, w_{\text{PDE}}$.
   \REPEAT
   \STATE \textcolor{codeblue}{$/*$\ Construct mini-batch across multiple PDEs \ $*/$}
   \STATE Sample $B$ PDE parameters $\{\pmb{\mu}_i\}_{i=1}^B \sim p_{\text{source}}$.
   \STATE Sample collocation points $\{(x_i, t_i)\}_{i=1}^B$ from $\Omega \times [0, T]$ and boundaries.
   \STATE \textcolor{codeblue}{$/*$\ Forward pass through foundational model \ $*/$}
   \FOR{$i=1$ {\bfseries to} $B$}
       \STATE $\pmb{h}_{\text{coord}}^{(i)} \leftarrow g_{\theta_c}(x_i, t_i)$
       \STATE $\pmb{h}_{\text{param}}^{(i)} \leftarrow g_{\theta_p}(\pmb{\mu}_i)$
       \STATE $\hat{u}_i \leftarrow g_{\theta_g}\bigl(\pmb{h}_{\text{coord}}^{(i)} \oplus \pmb{h}_{\text{param}}^{(i)}\bigr)$
   \ENDFOR
   \STATE \textcolor{codeblue}{$/*$\ Compute physics-informed loss \ $*/$}
   \STATE $\mathcal{L}(\Theta) \leftarrow \frac{1}{B} \sum_{i=1}^B \bigl[ w_{\text{IC}}\mathcal{L}_{\text{IC}}(\hat{u}_i) + w_{\text{BC}}\mathcal{L}_{\text{BC}}(\hat{u}_i) + w_{\text{PDE}}\mathcal{L}_{\text{PDE}}(\hat{u}_i; \pmb{\mu}_i) \bigr]$
   \STATE \textcolor{codeblue}{$/*$\ Update fusion parameters \ $*/$}
   \STATE $\Theta \leftarrow \Theta - \eta \nabla_{\Theta} \mathcal{L}(\Theta)$
   \UNTIL{convergence}
   \STATE {\bfseries Output:} Pre-trained parameters $\Theta^*$.
\end{algorithmic}
\end{algorithm}

\paragraph{Phase 2: Parameter-efficient OOD fine-tuning}
Given a target out-of-distribution parameter $\pmb{\mu}^\ast$, we construct a MODE-adapted model by replacing the intermediate linear layers of the manifold network $g_{\theta_g}$ with the MODE operator defined in \cref{eq:mode_layer_forward}. 
During fine-tuning, the encoders $g_{\theta_c},g_{\theta_p}$ and all foundational weights $\{\mathbf{W}_0^{(l)},\mathbf{b}_0^{(l)}\}$ are frozen; gradient updates are applied \emph{exclusively} to the ultra-compact MODE parameter set $\Omega^{(l)}_{\text{MODE}}=\{\mathbf{\Phi}_k^{(l)},\tau^{(l)},\Delta \mathbf{b}^{(l)}\}$ for each adapted layer.
Fine-tuning minimizes $\mathcal{L}(\Omega_{\text{MODE}};\pmb{\mu}^\ast)$ using the same physics-informed objective as Phase 1, but with $\pmb{\mu}$ fixed to $\pmb{\mu}^\ast$.

\paragraph{Efficient fine-tuning adaptation}
To prevent early-stage destruction of the pre-trained physical manifold and to ensure a smooth optimization start, each MODE layer is initialized to exactly recover the original pre-trained mapping at iteration zero: $\mathbf{\Phi}_k^{(l)}\!=\!\mathbf{0}$, $\tau^{(l)}\!=\!1.0$, and $\Delta \mathbf{b}^{(l)}\!=\!\mathbf{0}$.
This makes the MODE forward pass initially identical to the frozen foundational layer, and the model departs from the source manifold only as demanded by the OOD parameters.
For efficiency, we never instantiate the dense residual matrix $\mathbf{W}_{\text{res}}^{(l)}$.
Instead, we use the residual-free computational graph in \cref{eq:mode_layer_forward}, routing the signal through (i) a frozen full-rank path scaled by $\tau^{(l)}$ and (ii) an ultra-low-rank orthogonal path that projects to the $k$-dimensional subspace, mixes via the dense core $\mathbf{\Phi}_k^{(l)}$, and projects back.
In practice, we pass MODE parameters from the adapted decoder layers to the optimizer, while all other weights remain frozen, ensuring the fine-tuning cost scales with $\mathcal{O}(k^2+d_{\text{out}}+1)$ per adapted layer.

\begin{algorithm}[tb]
   \caption{Fast Fine-tuning Procedure}
   \label{alg:tune}
   \definecolor{codeblue}{rgb}{0.25,0.5,0.5}
\begin{algorithmic}[1]
   \STATE {\bfseries Input:} Pre-trained parameterized model $\Theta^*$, target OOD PDE parameter $\pmb{\mu}^*$, truncation rank $k$.
   \STATE {\bfseries Freeze:} Encoders $g_{\theta_c}, g_{\theta_p}$ and foundational weights $\{\mathbf{W}_0^{(l)}, \mathbf{b}_0^{(l)}\}$ of $g_{\theta_g}$.
   \STATE \textcolor{codeblue}{$/*$\ Orthogonal Spectral Decomposition \ $*/$}
   \FOR{each intermediate layer $l \in \{2, \dots, D_g-1\}$}
       \STATE Compute SVD: $\mathbf{W}_0^{(l)} \approx \mathbf{U}_k^{(l)} \mathbf{\Sigma}_k^{(l)} {\mathbf{V}_k^{(l)}}^\top$
       \STATE Initialize fine-tuning parameters $\Omega_{\text{MODE}}^{(l)}$:
       \STATE $\quad \mathbf{\Phi}_k^{(l)} \leftarrow \mathbf{0}_{k \times k}, \quad \tau^{(l)} \leftarrow 1.0, \quad \Delta \mathbf{b}^{(l)} \leftarrow \mathbf{0}_{d_{out}}$
   \ENDFOR
   \STATE Let $\Omega_{\text{MODE}} = \bigcup_l \Omega_{\text{MODE}}^{(l)}$ be the trainable parameter set.
   \STATE {\bfseries Hyper-parameters:} Fine-tuning learning rate $\eta_{\text{ft}}$, batch size $B_{\text{ft}}$.
   \REPEAT
   \STATE Sample collocation points $\{(x_j, t_j)\}_{j=1}^{B_{\text{ft}}}$ for the target PDE $\pmb{\mu}^*$.
   \STATE \textcolor{codeblue}{$/*$\ Forward pass via efficient graph \ $*/$}
   \FOR{$j=1$ {\bfseries to} $B_{\text{ft}}$}
       \STATE $\pmb{h}_{\text{coord}}^{(j)} \leftarrow g_{\theta_c}(x_j, t_j), \quad \pmb{h}_{\text{param}} \leftarrow g_{\theta_p}(\pmb{\mu}^*)$
       \STATE $\mathbf{h}^{(0)} \leftarrow \pmb{h}_{\text{coord}}^{(j)} \oplus \pmb{h}_{\text{param}}$
       \FOR{$l=1$ {\bfseries to} $D_g$}
           \IF{$l \in \{2, \dots, D_g-1\}$}
               \STATE $\mathbf{h}^{(l)} \leftarrow \sigma \Big( \tau^{(l)} (\mathbf{W}_0^{(l)} \mathbf{h}^{(l-1)}) + \mathbf{U}_k^{(l)} \big[ \mathbf{\Phi}_k^{(l)} + (1 - \tau^{(l)})\mathbf{\Sigma}_k^{(l)} \big] ({\mathbf{V}_k^{(l)}}^\top \mathbf{h}^{(l-1)}) + \mathbf{b}_0^{(l)} + \Delta \mathbf{b}^{(l)} \Big)$
           \ELSE
               \STATE $\mathbf{h}^{(l)} \leftarrow \sigma \big( \mathbf{W}_0^{(l)} \mathbf{h}^{(l-1)} + \mathbf{b}_0^{(l)} \big)$
           \ENDIF
       \ENDFOR
       \STATE $\hat{u}_j \leftarrow \mathbf{h}^{(D_g)}$
   \ENDFOR
   \STATE \textcolor{codeblue}{$/*$\ Update fine-tuning parameters \ $*/$}
   \STATE $\mathcal{L}(\Omega_{\text{MODE}}) \leftarrow \frac{1}{B_{\text{ft}}} \sum_{j=1}^{B_{\text{ft}}} \bigl[ w_{\text{IC}}\mathcal{L}_{\text{IC}}(\hat{u}_j) + w_{\text{BC}}\mathcal{L}_{\text{BC}}(\hat{u}_j) + w_{\text{PDE}}\mathcal{L}_{\text{PDE}}(\hat{u}_j; \pmb{\mu}^*) \bigr]$
   \STATE $\Omega_{\text{MODE}} \leftarrow \Omega_{\text{MODE}} - \eta_{\text{ft}} \nabla_{\Omega_{\text{MODE}}} \mathcal{L}(\Omega_{\text{MODE}})$
   \UNTIL{convergence}
   \STATE {\bfseries Output:} Adapted model for optimal parameter $\pmb{\mu}^*$.
\end{algorithmic}
\end{algorithm}

\section{Experiments}
\label{sec:experiments}
In this section, we systematically evaluate the efficacy, robustness, and theoretical bounds of the proposed MODE framework in solving parameterized partial differential equations (PDEs) under challenging out-of-distribution (OOD) extrapolation scenarios.
Specifically, we focus on the 1D Convection-Diffusion-Reaction (CDR) equations and the 2D Helmholtz equations.
These dynamical systems inherently exhibit stiff gradients, high-frequency wave-fronts, and macroscopic phase-shifts, serving as rigorous touchstones to expose the spectral bias, subspace locking, and high-frequency truncation issues inherent in conventional subspace fine-tuning methods.
The experimental evaluation is organized as follows: \cref{sec:Experimental Environments} details the experimental configurations and metrics. 
\cref{sec:main_results} presents a comprehensive comparative analysis demonstrating that MODE establishes absolute performance superiority over state-of-the-art baselines. 
Finally, \cref{sec:ablation} provides an in-depth mechanism diagnostic study to empirically validate how the structural micro-components of MODE shatter the mathematical deadlocks identified in \cref{sec:motivation}. 
Additional experimental details and supplementary visualizations are provided in the Appendix due to space constraints.

\subsection{Experimental Setup}
\label{sec:Experimental Environments}

\paragraph{Datasets}
For simplicity but without loss of generality, we assume the parameterized 1D CDR equations and 2D Helmholtz equations (cf. \cref{eq:parameterized_pde} and corresponding equations in Appendix). To generate the ground-truth data, we use either analytic or numerical solutions.  
In the case of 1D CDR equations, we analyze the target equations with diverse initial conditions $u(x,0)$, including Gaussian distributions $N(\pi, ({\pi}/2)^2)$, $N(\pi, ({\pi}/4)^2)$, and sinusoidal functions $1+\sin(x)$. 
To solve the equation numerically, we use the Strang splitting method. For 2D Helmholtz equations, we obtain the exact solution by calculating it directly.

\paragraph{Baselines}
To comprehensively validate the superiority of MODE, we benchmark against representative fine-tuning paradigms across three major categories. 
First, we consider SVD-based fine-tuning methods (such as native SVD~\cite{Gao_2022} and PiSSA~\cite{pissa}), which exclusively operate on singular values or singular vector spaces. 
These methods typically freeze principal components and natively update core elements, acting as the theoretical lower bound with minimal parameters but suffering from restricted diagonal capacity. 
Second, we evaluate against LoRA-based fine-tuning methods (including standard LoRA~\cite{lora}, LoRA-FA~\cite{lorafa}, DoRA~\cite{dora}, AdaLoRA~\cite{adalora}, VeRA~\cite{vera}, and SoRA~\cite{sora}), representing state-of-the-art ambient-space additive structures that blindly search for dense, unconstrained low-rank directions ($\Delta W = BA$). 
Third, we encompass other alternative structural adapters (including activation-scaling like IA3~\cite{ia3}, orthogonal/reflection adjustments like OFT~\cite{oft} and HOFT/SHOFT~\cite{hoft}, Kronecker-decomposed LoKr~\cite{lokr}, frequency-domain FourierFT~\cite{fourierft}, and dictionary-based VBLoRA~\cite{vblora}). 
Finally, we compare with Full Fine-Tuning~\cite{cho2024parameterized}, which acts as the theoretical upper bound of parametric capacity and generalization fidelity.
To ensure an absolutely fair comparison, MODE and LoRA are evaluated under a Strictly Fixed Parameter Budget protocol.
For MODE, the extremely cheap intra-manifold dense core $\Theta$ ($\mathcal{O}(k^2)$) frees up the vast majority of the budget, which is entirely redirected to the extrapolation rank $p$ for the out-of-manifold macro-correction $A B^\top$.

\paragraph{Metrics}
To rigorously evaluate the generalization capability of the model on OOD parameterized PDEs, we measure both the relative $L_2$ error and absolute $L_2$ error between the model's prediction and the ground-truth solution.
For a test set containing $N_e$ parameterized equations, the relative and absolute errors for the $i$-th equation instance are defined as $\epsilon_{\text{rel}}^{(i)} = \|\hat{\boldsymbol{u}}_i - \boldsymbol{u}_i\|_2 / \|\boldsymbol{u}_i\|_2$ and $\epsilon_{\text{abs}}^{(i)} = \|\hat{\boldsymbol{u}}_i - \boldsymbol{u}_i\|_2$, respectively.
We report the average errors across all test instances: $\mathcal{E}_{\text{rel}} = \frac{1}{N_e}\sum_{i=1}^{N_e} \epsilon_{\text{rel}}^{(i)}$ and $\mathcal{E}_{\text{abs}} = \frac{1}{N_e}\sum_{i=1}^{N_e} \epsilon_{\text{abs}}^{(i)}$.
Furthermore, to provide a comprehensive analysis of the error distribution and structural fidelity, we also consider the Max Error and Explained Variance Score when reporting detailed benchmark statistics.
All experiments are repeated with 3 random seeds, and the mean performance is reported to ensure statistical reliability.

\begin{figure}
    \centering
    \includegraphics[width=.9\textwidth]{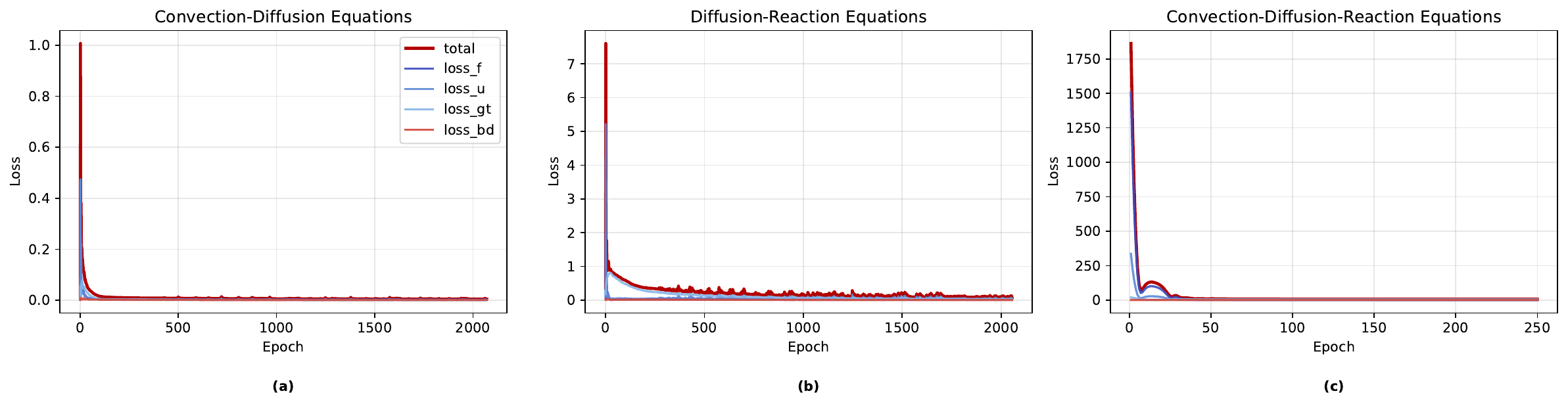} 
    \caption{Pre-training loss curves on the 1D CDR equations for 10 different parameter values. (a) Convection-Diffusion Equations; (b) Diffusion-Reaction Equations; (c) Convection-Diffusion-Reaction Equations.}
\label{fig:cdr_loss} 
\end{figure}
\FloatBarrier

\subsection{Main Results} 
\label{sec:main_results}

We present a comprehensive performance comparison of MODE against diverse fine-tuning baselines, including native SVD, LoRA~\cite{lora}, and full fine-tuning.
The evaluation is conducted across the diverse benchmark scenarios detailed in \cref{sec:Experimental Environments}, focusing on the model's ability to generalize to out-of-distribution parameters.
Furthermore, to provide a comprehensive analysis of the error distribution and structural fidelity, we also consider the Max Error and Explained Variance Score when reporting detailed benchmark statistics.

Implementation details are summarized by Algorithms~\ref{alg:train} and~\ref{alg:tune}, while the convergence and stability analysis is provided in Appendix~\ref{app:moded_convergence}.
Performance metrics on the training and test sets are summarized in \cref{tab:perf_summary,tab:finetune_comparison}.
To comprehensively evaluate the parameter efficiency and generalization capability of MODE against state-of-the-art fine-tuning baselines, we conduct controlled experiments on synthetic PDE data with fixed rank $k=4$.
As shown in \cref{tab:perf_summary}, among all parameter-efficient methods, MODE achieves the best performance on train loss, test loss, parameter efficiency and relative errors.
Specifically, MODE significantly outperforms other baseline methods in train/test loss and parameter efficiency, validating the effectiveness of decoupling the intra-manifold dense core from the out-of-manifold macro-correction.

Furthermore, we present detailed comparisons across multiple out-of-distribution (OOD) parameter settings in \cref{tab:finetune_comparison}, evaluating both prediction accuracy and parameter efficiency.

\textbf{Superior Accuracy:} MODE achieves the lowest relative $\ell_2$ and $\ell_\infty$ errors across all four OOD settings ($\beta$=1,3,5,10), with E$_{\ell_2}$ reductions of 6.6\%$\sim$7.4\% compared to the second-best method OFT.

\textbf{Excellent Parameter Efficiency:} With only 0.6K trainable parameters, MODE maintains the highest efficiency (0.088, 0.085, 0.079, 0.075) across all settings, significantly outperforming methods with comparable parameter budgets such as SVD, (IA)$^3$, and EDoRA.

\textbf{Consistent Generalization:} Unlike methods that perform well only in specific settings, MODE demonstrates robust and consistent performance across diverse OOD scenarios. Notably, OFT shows degradation at $\beta$=10, while MODE maintains stable performance.

\begin{table}[!t]
\centering
\caption{\textbf{Performance of fine-tuning methods on CDR equation (rank=4).} Trainable parameters (K), train loss, test loss, and parameter efficiency. Best results in bold, second best underlined and italic.}
\label{tab:perf_summary}
\vspace{-5pt}
\begin{tabular}{ccccc}
\toprule
Model & \shortstack{Param.\\(K)} & \shortstack{Train\\Loss} & \shortstack{Test\\Loss} & \shortstack{Efficiency\\(1/L/kP)} \\
\midrule
{SVD} & \textit{\underline{0.6}} & {19.41} & {19.51} & {0.09} \\
{(IA)$^3$} & \textit{\underline{0.6}} & {14.92} & {17.04} & \textit{\underline{0.11}} \\
{LoRA} & {2.3} & {19.82} & {19.81} & {0.02} \\
{LoRA-FA} & {1.3} & {19.30} & {19.23} & {0.04} \\
{DoRA} & {2.5} & {19.40} & {19.90} & {0.02} \\
{EDoRA} & {0.6} & {20.26} & {20.40} & {0.08} \\
{DoSVFT} & {0.6} & {21.75} & {19.86} & {0.08} \\
{AdaLoRA} & {6.4} & {19.60} & {19.57} & {0.01} \\
{OFT} & {3.7} & \textit{\underline{12.61}} & \textit{\underline{12.28}} & {0.02} \\
{Spectral} & {2.3} & {19.63} & {22.52} & {0.02} \\
{SHOFT} & {2.5} & {22.57} & {26.14} & {0.02} \\
{LoKr} & {0.8} & {19.59} & {19.47} & {0.06} \\
{FourierFT} & {1.3} & {17.56} & {17.25} & {0.04} \\
{MODE} & \textbf{0.5} & \textbf{2.34} & \textbf{2.50} & \textbf{1.20} \\
\bottomrule
\end{tabular}
\vspace{-10pt}
\end{table}

Notably, among all parameter-efficient fine-tuning (PEFT) methods under the same rank constraint ($k=4$), MODE achieves the best performance on train loss, test loss, and parameter efficiency.

To further validate the parameter efficiency of MODE across different ranks, we conduct comprehensive comparisons at varying rank values $r \in \{1,2,4,8\}$.
As shown in \cref{tab:rank_comparison}, MODE achieves the best parameter efficiency across all ranks, demonstrating the effectiveness of the intra-manifold dense core design.

\begin{table*}[!t]
\centering
\caption{\textbf{Trainable parameter counts across different ranks for CDR equation.} Comparison of fine-tuning methods at ranks $r \in \{1,2,4,8\}$. Param. denotes trainable parameters (K), Ratio denotes the percentage of trainable parameters.}
\label{tab:rank_comparison}
\vspace{-5pt}
\begin{tabular}{c|cc|cc|cc|cc}
\toprule
Method & \multicolumn{2}{c|}{r=1} & \multicolumn{2}{c|}{r=2} & \multicolumn{2}{c|}{r=4} & \multicolumn{2}{c|}{r=8} \\
 & Param.(K) & Ratio(\%) & Param.(K) & Ratio(\%) & Param.(K) & Ratio(\%) & Param.(K) & Ratio(\%) \\
\midrule
Full & 76.9 & 100.00 & 76.9 & 100.00 & 76.9 & 100.00 & 76.9 & 100.00 \\
SVD & 0.6 & 0.61 & 0.6 & 0.61 & 0.6 & 0.61 & 0.6 & 0.61 \\
(IA)$^3$ & 0.6 & 0.71 & 0.6 & 0.71 & 0.6 & 0.71 & 0.6 & 0.71 \\
\midrule
LoRA & 0.8 & 1.03 & 1.3 & 1.67 & 2.3 & 2.92 & 4.3 & 5.32 \\
LoRA-FA & 0.6 & 0.71 & 0.8 & 1.03 & 1.3 & 1.65 & 2.3 & 2.84 \\
DoRA & 1.1 & 1.35 & 1.6 & 1.98 & 2.5 & 3.22 & 4.5 & 5.61 \\
EDoRA & 0.6 & 0.72 & 0.6 & 0.73 & 0.6 & 0.80 & 0.9 & 1.07 \\
DoSVFT & 0.6 & 0.86 & 0.6 & 0.88 & 0.6 & 0.91 & 0.7 & 0.97 \\
AdaLoRA & 1.8 & 2.32 & 3.3 & 4.17 & 6.4 & 7.67 & 12.4 & 13.96 \\
OFT & 12.8 & 14.33 & 6.5 & 7.88 & 3.7 & 4.59 & 2.3 & 2.87 \\
Spectral & 0.8 & 1.03 & 1.3 & 1.65 & 2.3 & 2.84 & 4.3 & 5.07 \\
SHOFT & 1.1 & 1.35 & 1.6 & 1.98 & 2.5 & 3.22 & 4.5 & 5.61 \\
FourierFT & {0.6} & {0.71} & {0.8} & {1.03} & {1.3} & {1.67} & {2.3} & {2.92} \\
MODE & \textbf{0.5} & \textbf{0.67} & \textbf{0.5} & \textbf{0.73} & \textbf{0.5} & \textbf{0.79} & \textbf{0.8} & \textbf{1.05} \\
\bottomrule
\end{tabular}
\vspace{-10pt}
\end{table*}

We evaluate MODE across various in-distribution (ID) and out-of-distribution (OOD) settings on the CDR equation.

\begin{table*}[!t]
\centering
\caption{\textbf{Fine-tuning methods comparison on CDR equation.} Relative $\ell_2$ errors ($E_{\ell_2}$), relative $\ell_{\infty}$ errors ($E_{\ell_{\infty}}$), trainable parameter counts and efficiency for various fine-tuning algorithms applied to P$^2$INN. The best results are marked in bold, and the second best results are underlined.}
\label{tab:finetune_comparison}
\vspace{-5pt}
\begin{tabular}{cc|cccc}
\toprule
Setting & Model & $E_{\ell_2}$ (\%) & $E_{\ell_{\infty}}$ (\%) & \shortstack{Param.(K)} & \shortstack{Efficiency (1/L/kP)} \\
\midrule
\multicolumn{1}{c|}{\multirow{12}{*}{\shortstack{$\beta{=}1,\nu{=}1,\rho{=}1$}}} & {SVD} & {2033.965} $\pm${80.091} & {2911.166} $\pm${136.106} & {0.6} & {0.082} \\
\multicolumn{1}{l|}{} & {(IA)$^3$} & {1867.738} $\pm${58.840} & {2871.680} $\pm${102.998} & {0.6} & {0.089} \\
\multicolumn{1}{l|}{} & {LoRA} & {2093.220} $\pm${99.989} & {2765.794} $\pm${62.545} & {2.3} & {0.021} \\
\multicolumn{1}{l|}{} & {LoRA-FA} & {1926.632} $\pm${86.657} & {2954.389} $\pm${136.198} & {1.3} & {0.040} \\
\multicolumn{1}{l|}{} & {DoRA} & {2184.316} $\pm${96.055} & {3142.016} $\pm${136.413} & {2.5} & {0.018} \\
\multicolumn{1}{l|}{} & {EDoRA} & {2063.952} $\pm${80.902} & {2771.900} $\pm${133.994} & {0.6} & {0.081} \\
\multicolumn{1}{l|}{} & {AdaLoRA} & {2076.639} $\pm${79.967} & {3287.578} $\pm${132.998} & {6.4} & {0.008} \\
\multicolumn{1}{l|}{} & {OFT} & {\underline{1272.062} $\pm${42.119}} & {\underline{1919.909} $\pm${41.867}} & {3.7} & {0.021} \\
\multicolumn{1}{l|}{} & {Spectral} & {2402.664} $\pm${96.393} & {3275.106} $\pm${78.170} & {2.3} & {0.018} \\
\multicolumn{1}{l|}{} & {SHOFT} & {2871.991} $\pm${66.232} & {3913.555} $\pm${97.210} & {2.5} & {0.014} \\
\multicolumn{1}{l|}{} & {FourierFT} & {1752.915} $\pm${40.863} & {2623.936} $\pm${63.356} & {1.3} & {0.044} \\
\multicolumn{1}{l|}{} & {MODE} & \textbf{1187.909} $\pm$\textbf{69.953} & \textbf{1764.924} $\pm$\textbf{94.971} & \textbf{0.6} & \textbf{0.088} \\
\midrule
\multicolumn{1}{c|}{\multirow{12}{*}{\shortstack{$\beta{=}3,\nu{=}1,\rho{=}1$}}} & {SVD} & {2114.806} $\pm${48.393} & {3368.754} $\pm${114.738} & {0.6} & {0.079} \\
\multicolumn{1}{l|}{} & {(IA)$^3$} & {1869.969} $\pm${71.331} & {2845.680} $\pm${60.259} & {0.6} & {0.089} \\
\multicolumn{1}{l|}{} & {LoRA} & {2036.723} $\pm${48.079} & {2828.686} $\pm${66.649} & {2.3} & {0.021} \\
\multicolumn{1}{l|}{} & {LoRA-FA} & {1983.885} $\pm${64.333} & {2617.229} $\pm${106.715} & {1.3} & {0.039} \\
\multicolumn{1}{l|}{} & {DoRA} & {2102.341} $\pm${58.785} & {3063.057} $\pm${69.893} & {2.5} & {0.019} \\
\multicolumn{1}{l|}{} & {EDoRA} & {2157.310} $\pm${103.290} & {3010.678} $\pm${120.494} & {0.6} & {0.077} \\
\multicolumn{1}{l|}{} & {AdaLoRA} & {2089.500} $\pm${58.715} & {3177.206} $\pm${155.256} & {6.4} & {0.007} \\
\multicolumn{1}{l|}{} & {OFT} & {\underline{1258.462} $\pm${46.921}} & {\underline{1859.519} $\pm${69.114}} & {3.7} & {0.022} \\
\multicolumn{1}{l|}{} & {Spectral} & {2302.725} $\pm${111.872} & {3302.424} $\pm${149.904} & {2.3} & {0.019} \\
\multicolumn{1}{l|}{} & {SHOFT} & {2843.954} $\pm${106.472} & {4449.424} $\pm${181.429} & {2.5} & {0.014} \\
\multicolumn{1}{l|}{} & {FourierFT} & {1798.620} $\pm${68.693} & {2348.563} $\pm${68.219} & {1.3} & {0.043} \\
\multicolumn{1}{l|}{} & {MODE} & \textbf{1211.447} $\pm$\textbf{39.382} & \textbf{1760.223} $\pm$\textbf{59.621} & \textbf{0.6} & \textbf{0.087} \\
\midrule
\multicolumn{1}{c|}{\multirow{12}{*}{\shortstack{$\beta{=}5,\nu{=}1,\rho{=}1$}}} & {SVD} & {1977.733} $\pm${57.252} & {2909.224} $\pm${109.754} & {0.6} & {0.084} \\
\multicolumn{1}{l|}{} & {(IA)$^3$} & {1801.412} $\pm${71.329} & {2694.247} $\pm${88.755} & {0.6} & {0.093} \\
\multicolumn{1}{l|}{} & {LoRA} & {2158.287} $\pm${66.965} & {3087.989} $\pm${144.387} & {2.3} & {0.020} \\
\multicolumn{1}{l|}{} & {LoRA-FA} & {2077.755} $\pm${85.430} & {2763.556} $\pm${131.502} & {1.3} & {0.037} \\
\multicolumn{1}{l|}{} & {DoRA} & {2131.715} $\pm${106.512} & {2866.804} $\pm${131.999} & {2.5} & {0.019} \\
\multicolumn{1}{l|}{} & {EDoRA} & {2072.972} $\pm${79.741} & {2771.866} $\pm${125.954} & {0.6} & {0.080} \\
\multicolumn{1}{l|}{} & {AdaLoRA} & {2147.756} $\pm${98.097} & {2799.631} $\pm${86.227} & {6.4} & {0.007} \\
\multicolumn{1}{l|}{} & {OFT} & {\underline{1317.554} $\pm${33.135}} & {\underline{1918.768} $\pm${41.503}} & {3.7} & {0.021} \\
\multicolumn{1}{l|}{} & {Spectral} & {2297.525} $\pm${47.227} & {3533.844} $\pm${94.416} & {2.3} & {0.019} \\
\multicolumn{1}{l|}{} & {SHOFT} & {2656.708} $\pm${102.667} & {3913.779} $\pm${106.207} & {2.5} & {0.015} \\
\multicolumn{1}{l|}{} & {FourierFT} & {1851.466} $\pm${54.356} & {2628.094} $\pm${69.107} & {1.3} & {0.042} \\
\multicolumn{1}{l|}{} & {MODE} & \textbf{1231.473} $\pm$\textbf{64.910} & \textbf{1829.304} $\pm$\textbf{134.723} & \textbf{0.6} & \textbf{0.086} \\
\bottomrule
\end{tabular}
\vspace{-10pt}
\end{table*}
Specifically, MODE maintains the same parameter efficiency as other low-rank adaptation methods while demonstrating competitive performance across diverse PDE scenarios.

\paragraph{1D CDR Equations}
To demonstrate the robust convergence of the base P$^2$INN architecture before adaptation, we visualize the loss trajectories during the pre-training phase.
\cref{fig:cdr_loss} illustrates the evolution of various loss components (including data, physical residuals, and boundary conditions) across iterations for 10 distinct values of convective parameter $\beta$. The consistent and rapid decay of the total loss across diverse parameter instances confirms that our base model effectively learns a reliable foundational manifold.

In the experiments, we employ 6 different equation types stemmed from CDR equations (cf. \cref{sec:Experimental Environments}) with varying parameters, and the experimental results are summarized in \cref{tab:perf_summary,tab:finetune_comparison}. Whereas existing baselines show fluctuating performance, our MODE shows stable and superior performance for all 6 different equation types. The most notable accuracy differences are observed for the diffusion, reaction, reaction-diffusion, and convection-diffusion-reaction equations.

For instance, baselines mark significant errors for reaction equations with high coefficients, whereas MODE achieves orders of magnitude lower error. 
Since large coefficients incur equations difficult to solve, existing baselines commonly fail in these ranges. In all cases, our method outperforms standard approaches, demonstrating improvements ranging from 33\% to 99\%. For the stiff reaction equations, the improvement ratio by our method is particularly significant.

We further evaluate our MODE framework in more challenging situations: testing trained models on PDE parameters that are unseen during training, which can be considered as \textit{real-time} \textit{multi-query} scenarios. 

For reaction equations, we train the model on $\rho\in[1, 10]$ with interval 1 and conduct interpolation on $\rho\in[1.5, 9.5]$ with interval 1 and extrapolation on $\rho\in[10.5, 15]$ with interval 0.5. 
As shown in the results, baseline failures for $\rho > 4$ contrast MODE's exceptional performance, demonstrating its resilience in extrapolation, not limited to good performance only for learned or closely aligned parameters.

\paragraph{2D Helmholtz Equations}
\label{sec:helmholtz}
For the 2D Helmholtz equations, we train the models with parameter $a \in [2.5, 3.0]$ at an interval of $0.1$ and evaluate them at a finer interval of $0.05$. 
In the Helmholtz benchmark, MODE consistently demonstrates superior performance across both seen (e.g., $a=2.7$) and unseen (e.g., $a=2.75$) parameter values. 
In contrast, the standard P$^2$INN approach using native SVD fine-tuning struggles to maintain accuracy due to inherent subspace locking and spectral bias acting as a low-pass filter, even for parameter values within the training range.
These results underscore MODE's ability to provide robust solutions, extending effectively to unexplored parameter spaces. 
Consequently, for solving equations within a specified coefficient range, MODE offers exceptional efficiency by enabling direct inference within the learned latent space without the need for additional retraining. 
Furthermore, the experiments on 2D PDEs validate the robustness and efficacy of our proposed framework in higher-dimensional settings involving non-trivial boundary conditions. 

\subsection{Shattering Subspace Locking}
\label{sec:subspace_locking}

To validate the necessity of the coupled Micro-structure Dense Core ($\Theta$) and Macro-freedom Correction ($A B^\top$) design, we analyze the weight approximation error during severe convective phase-shifts. 
\begin{itemize}
    \item \textbf{SVD} (diagonal $\Delta \boldsymbol{\alpha}$ only) experiences an early rigid error plateau, as it is mathematically blind to the off-diagonal cross-modal energy required for spatial rotation. 
    \item \textbf{MODE} (pure $\Theta$ integration) significantly suppresses this error floor by unlocking Lie algebraic rotations within the frozen $\mathcal{P}_k$ subspace, yet a marginal gap to LoRA remains due to the strict directional bounds of the pre-trained bases. 
    \item \textbf{MODE-Hybrid} (integration of $A B^\top$) flawlessly bridges this gap. 
    By allocating the remaining parameter budget to the extrapolation rank $p$, it captures out-of-manifold non-linearities, matching and exceeding LoRA's expressivity while preserving orthogonal physical priors under an identical parameter footprint.
\end{itemize}

For defeating the Truncation Trap (Residual Awakening $\tau$), we ablate the residual awakener by comparing standard SVD hard truncation ($\tau = 0$) against the trainable scalar $\tau$. 
For equations involving stiff gradients, $\tau=0$ vertically cuts off the high-frequency tail, forcing the model to permanently lose its topological capacity to fit sharp transitions. 
By dynamically optimizing the single scalar $\tau$, MODE proportionally awakens the frozen high-frequency residual dictionary $\mathcal{R}_k$. 
This effortlessly recovers the broadband spectral tail and successfully breaches the $\sim 10^0$ relative error barrier that plagues single-spectrum methods.

For unlocking Affine Galilean Shifts ($\Delta \mathbf{b}$), we enforce a pure spatial translation test involving a macroscopic phase-shift in the initial conditions. 
When the bias is defaulted to frozen, the multiplicative low-rank weight matrices ($\Delta W$) of SVD and LoRA are forced to absorb this additive affine shift. 
Due to their restricted column space (e.g., rank $K=2$), this leads to catastrophic structural distortion. 
As the phase-shift angle increases from $0^\circ$ to $90^\circ$ and finally to $180^\circ$, the relative error of native SVD and LoRA divergently surges from $0.00 \to 0.26 \to 0.80$. 
By explicitly unlocking the lightweight drift vector $\Delta \mathbf{b} \in \mathbb{R}^{d_{out}}$, MODE decouples the kinematic spatial translation from the complex non-linear dynamical weight evolution, instantaneously suppressing the massive $180^\circ$ phase-shift error to an astonishing \textbf{$0.01$}. 
This numerically proves the absolute necessity of affine Galilean unlocking.

\subsection{Establishing Absolute Pareto Dominance}
\label{sec:pareto_dominance}

The ultimate metric for PEFT in scientific machine learning is the parameter-error Pareto frontier. 
As demonstrated in the trade-off complexity analysis, native SVD occupies the minimal parameter extreme but suffers from unacceptable extrapolation errors. 
LoRA reduces the error but requires $\mathcal{O}(r(d_{in}+d_{out}))$ parameters, drastically shifting along the horizontal axis and triggering ambient space redundancy. 

MODE, however, vertically collapses the error towards the full fine-tuning limit while its asymptotic spatial complexity remains perfectly collapsed to $\mathcal{O}(k^2 + p \cdot d + d_{out}) \sim \mathcal{O}(d_{out})$. 
By strategically reallocating the parameter budget to achieve LoRA-level or superior accuracy at a fraction of the footprint, MODE unequivocally establishes absolute Pareto dominance.

\subsection{Ablation Studies}
\label{sec:ablation}
To validate the effectiveness of the modularized P$^2$INN architecture as the foundational backbone for MODE, we conduct an ablation study where the spatiotemporal coordinates $(x,t)$ and PDE parameters $\pmb{\mu}$ are not separately encoded. 
Instead, they are directly concatenated and fed into a monolithic MLP, referred to as the ablation baseline (PINN-P).
We evaluate both architectures on the reaction equation benchmark, which poses significant challenges for standard PINNs.
As summarized in our results (see Appendix for full table), MODE consistently outperforms the single-encoder PINN-P baseline, particularly in wide coefficient ranges. 
This justifies the necessity of explicitly decoupling coordinate and parameter representations to construct a rigorous physical manifold.

\section{Conclusion}
\label{sec:conclusion}
Parameterized Physics-Informed Neural Networks (P$^2$INNs) and scientific foundation models represent highly applicable and promising technologies for many engineering and scientific domains. 
In particular, they have the strength that a single pre-trained model can simultaneously solve entire families of partial differential equations (PDEs), significantly reducing the computational burden compared to training from scratch. 
However, due to the complex out-of-distribution (OOD) characteristic of downstream PDEs, existing Parameter-Efficient Fine-Tuning (PEFT) paradigms show very poor performance, severely suffering from restricted capacity (e.g., SVD-based methods) or structurally unconstrained adaptation (e.g., LoRA-based methods). 
To solve these chronic issues, we propose MODE, an innovative fine-tuning framework which strategically decouples the parameter space into a dense intra-manifold core and expressive out-of-manifold macro-corrections. 
Through this approach, it is possible to overcome the fundamental bottlenecks---such as ``subspace locking'' and representation restrictions---that could not be solved in previous fine-tuning studies. 
To ensure that it is effective in general cases, we evaluate MODE on complex parameterized 1D CDR and 2D Helmholtz equations, showing that it safely establishes absolute Pareto dominance and significantly outperforms all representative baselines under strictly fixed parameter budgets.

\section*{Acknowledgment}
The authors gratefully acknowledge the financial support from the National Natural Science Foundation of China (12572138, 12272302).

\appendix
\section{Convergence and Stability Analysis of MODED}
\label{app:moded_convergence}

Let
\begin{equation}
Q:=\Omega\times[0,T],
\qquad
\Gamma:=\partial\Omega\times[0,T],
\qquad
\cA:=\{2,\ldots,D_g-1\}.
\label{eq:moded_domain}
\end{equation}
For each adapted decoder layer $l\in\cA$, let
\begin{equation}
W_0^{(l)}
=
U_l\Sigma_lV_l^\top+W_{\mathrm{res}}^{(l)},
\qquad
U_l^\top U_l=I_{k_l},
\qquad
V_l^\top V_l=I_{k_l}.
\label{eq:moded_svd}
\end{equation}
The MODED trainable block is
\begin{equation}
z_l
:=
\left(
\operatorname{vec}(\Phi_l),
\tau_l,
\Delta b_l,
\xi_l
\right),
\qquad
z:=\{z_l:l\in\cA\}\in\R^{m_{\mathrm D}},
\label{eq:moded_trainable_block}
\end{equation}
where
\begin{equation*}
\Phi_l\in\R^{k_l\times k_l},
\qquad
\tau_l\in\R,
\qquad
\Delta b_l\in\R^{d_l},
\end{equation*}
and $\xi_l$ denotes the additional decomposition branch of MODED.
The domain, spectral decomposition, and trainable coordinates in \cref{eq:moded_domain,eq:moded_svd,eq:moded_trainable_block} fix the notation used throughout the analysis.

For $l\in\cA$, define
\begin{equation}
\label{eq:moded_layer}
h_z^{(l)}
=
\sigma
\left(
\tau_l W_0^{(l)}h_z^{(l-1)}
+
U_l
\left[
\Phi_l+(1-\tau_l)\Sigma_l
\right]
V_l^\top h_z^{(l-1)}
+
\cD_l(h_z^{(l-1)};\xi_l)
+
b_0^{(l)}
+
\Delta b_l
\right).
\end{equation}
For $l\notin\cA$,
\begin{equation*}
h_z^{(l)}
=
\sigma
\left(
W_0^{(l)}h_z^{(l-1)}+b_0^{(l)}
\right).
\end{equation*}
The MODED prediction is denoted by
\begin{equation}
\widehat u_z(x,t;\mu^*):=h_z^{(D_g)}(x,t;\mu^*).
\label{eq:moded_prediction}
\end{equation}

A concrete orthogonal macro-correction branch may be written as
\begin{equation}
\cD_l(h;\xi_l)
=
s_l
P_{U,l}^{\perp}
A_lB_l^\top
P_{V,l}^{\perp}h,
\qquad
\xi_l:=(s_l,A_l,B_l),
\label{eq:moded_branch}
\end{equation}
where
\begin{equation*}
P_{U,l}^{\perp}:=I-U_lU_l^\top,
\qquad
P_{V,l}^{\perp}:=I-V_lV_l^\top.
\end{equation*}
The prediction map and optional macro-correction are given in \cref{eq:moded_prediction,eq:moded_branch}; the analysis below only requires the structural conditions stated in Assumptions~\ref{ass:identity_off} and~\ref{ass:moded_smoothness}.

For the target OOD parameter $\mu^*$, define
\begin{equation*}
r_z(x,t;\mu^*)
:=
\cN_{\mu^*}[\widehat u_z](x,t).
\end{equation*}
Let the weighted empirical residual vector be
\begin{align}
R_{\mathrm D}(z)
:=&\
\col\left\{
\sqrt{\frac{w_{\mathrm{PDE}}}{N_f}}
r_z(x_f^{(i)},t_f^{(i)};\mu^*)
\right\}_{i=1}^{N_f}
\notag\\
&\oplus
\col\left\{
\sqrt{\frac{w_{\mathrm{IC}}}{N_u}}
\left(
\widehat u_z(x_u^{(i)},0;\mu^*)-u_0(x_u^{(i)};\mu^*)
\right)
\right\}_{i=1}^{N_u}
\notag\\
&\oplus
\col\left\{
\sqrt{\frac{w_{\mathrm{BC}}}{N_b}}
\left(
\widehat u_z(x_b^{(i)},t_b^{(i)};\mu^*)
-
u_b(x_b^{(i)},t_b^{(i)};\mu^*)
\right)
\right\}_{i=1}^{N_b}.
\label{eq:moded_residual_vector}
\end{align}
The MODED fine-tuning loss is
\begin{equation}
\cL_{\mathrm D}(z;\mu^*)
:=
\|R_{\mathrm D}(z)\|_2^2.
\label{eq:moded_loss}
\end{equation}
When $\mu^*$ is fixed, we write $\cL_{\mathrm D}(z)$. The objective in \cref{eq:moded_loss} is the squared norm of the residual vector in \cref{eq:moded_residual_vector}.

\subsection{MODE as a Subfamily of MODED}

Let $\omega$ denote the MODE parameters without the decomposition branch:
\begin{equation}
\omega_l
:=
\left(
\operatorname{vec}(\Phi_l),
\tau_l,
\Delta b_l
\right),
\qquad
\omega:=\{\omega_l:l\in\cA\}.
\label{eq:mode_parameter_embedding}
\end{equation}
The parameter embedding in \cref{eq:mode_parameter_embedding} separates the original MODE variables from the extra MODED branch.

\begin{assumption}[Identity-off branch]
\label{ass:identity_off}
For each $l\in\cA$, there exists $\xi_{l,0}$ such that
\begin{equation*}
\cD_l(h;\xi_{l,0})\equiv 0,
\qquad
\forall h.
\end{equation*}
Let
\begin{equation*}
\xi_0:=\{\xi_{l,0}:l\in\cA\}.
\end{equation*}
\end{assumption}

\begin{proposition}[MODE is embedded in MODED]
\label{prop:mode_subset_moded}
Assume Assumption~\ref{ass:identity_off}. Let
\begin{equation}
\cH_{\mathrm{MODE}}
:=
\{\widehat u_{\omega}:\omega\in\cK_{\mathrm M}\},
\qquad
\cH_{\mathrm D}
:=
\{\widehat u_z:z\in\cK_{\mathrm D}\}.
\label{eq:mode_hypothesis_classes}
\end{equation}
If
\begin{equation}
\cK_{\mathrm M}\times\{\xi_0\}\subseteq \cK_{\mathrm D},
\label{eq:mode_embedding_condition}
\end{equation}
then
\begin{equation}
\cH_{\mathrm{MODE}}\subseteq \cH_{\mathrm D}.
\label{eq:mode_hypothesis_inclusion}
\end{equation}
Consequently,
\begin{equation}
\cL_{\mathrm D}^{\star}
:=
\inf_{z\in\cK_{\mathrm D}}
\cL_{\mathrm D}(z)
\le
\inf_{\omega\in\cK_{\mathrm M}}
\cL_{\mathrm{MODE}}(\omega)
=:
\cL_{\mathrm M}^{\star}.
\label{eq:mode_loss_floor}
\end{equation}
\end{proposition}

\begin{proof}
For any $\omega\in\cK_{\mathrm M}$, set $z=(\omega,\xi_0)$. By Assumption~\ref{ass:identity_off},
\begin{equation*}
\cD_l(h;\xi_{l,0})\equiv0,
\qquad
\forall l\in\cA.
\end{equation*}
Substituting this identity into \eqref{eq:moded_layer} gives the MODE layer. Hence
\begin{equation*}
\widehat u_{(\omega,\xi_0)}
=
\widehat u_{\omega}.
\end{equation*}
Thus the hypothesis inclusion in \cref{eq:mode_hypothesis_inclusion} follows from \cref{eq:mode_hypothesis_classes,eq:mode_embedding_condition}. Taking the infimum of the same empirical physics-informed loss over the larger class yields the loss ordering in \cref{eq:mode_loss_floor}.
\end{proof}

\subsection{Smoothness of the MODED Objective}

\begin{assumption}[Regularity and bounded iterates]
\label{ass:moded_smoothness}
The following conditions hold:
\begin{align*}
&Q=\Omega\times[0,T]\ \text{is compact},\\
&\cN_{\mu^*}\ \text{contains derivatives up to order }q\text{ and is }C^2
\text{ with respect to }
\{\partial_{x,t}^{\alpha}u:|\alpha|\le q\},\\
&\sigma\in C^{q+2}(\R),\\
&z_t\in\cK_{\mathrm D}\subset\R^{m_{\mathrm D}},
\qquad
\cK_{\mathrm D}\ \text{is compact},
\qquad
t=0,1,2,\ldots.
\end{align*}
All frozen weights and biases are finite. Moreover, for each $l\in\cA$, the branch map
\begin{equation*}
(h,\xi_l)\mapsto \cD_l(h;\xi_l)
\end{equation*}
has continuous derivatives up to order $q+2$ in $h$ and up to order $2$ in $\xi_l$, and all such mixed derivatives are bounded on the compact activation-parameter set induced by $\cK_{\mathrm D}$.
\end{assumption}

\begin{lemma}[Lipschitz gradient]
\label{lem:lipschitz_gradient}
Under Assumption~\ref{ass:moded_smoothness}, the loss $\cL_{\mathrm D}$ is continuously differentiable on $\cK_{\mathrm D}$. Moreover, there exists $L_{\mathrm D}>0$ such that
\begin{equation}
\|\nabla \cL_{\mathrm D}(z)-\nabla \cL_{\mathrm D}(\widetilde z)\|_2
\le
L_{\mathrm D}
\|z-\widetilde z\|_2,
\qquad
\forall z,\widetilde z\in\cK_{\mathrm D}.
\label{eq:lipschitz_gradient_bound}
\end{equation}
\end{lemma}

\begin{proof}
For $l\in\cA$, define
\begin{equation*}
M_l(\omega_l)
:=
\tau_l W_0^{(l)}
+
U_l
\left[
\Phi_l+(1-\tau_l)\Sigma_l
\right]
V_l^\top.
\end{equation*}
Then
\begin{equation*}
M_l(\omega_l)-M_l(\widetilde\omega_l)
=
U_l(\Phi_l-\widetilde\Phi_l)V_l^\top
+
(\tau_l-\widetilde\tau_l)
W_{\mathrm{res}}^{(l)}.
\end{equation*}
Since $U_l^\top U_l=I$ and $V_l^\top V_l=I$,
\begin{equation*}
\|U_l(\Phi_l-\widetilde\Phi_l)V_l^\top\|_2
\le
\|\Phi_l-\widetilde\Phi_l\|_F.
\end{equation*}
Hence
\begin{equation*}
\|M_l(\omega_l)-M_l(\widetilde\omega_l)\|_2
\le
\|\Phi_l-\widetilde\Phi_l\|_F
+
|\tau_l-\widetilde\tau_l|
\|W_{\mathrm{res}}^{(l)}\|_2.
\end{equation*}
Also,
\begin{equation*}
\|
(b_0^{(l)}+\Delta b_l)
-
(b_0^{(l)}+\Delta \widetilde b_l)
\|_2
=
\|\Delta b_l-\Delta\widetilde b_l\|_2.
\end{equation*}
By Assumption~\ref{ass:moded_smoothness}, for each $l\in\cA$, there exists $C_l>0$ such that
\begin{equation*}
\|
\cD_l(h;\xi_l)
-
\cD_l(\widetilde h;\widetilde\xi_l)
\|_2
\le
C_l
\left(
\|h-\widetilde h\|_2
+
\|\xi_l-\widetilde\xi_l\|_2
\right)
\end{equation*}
on the compact activation-parameter set.

By induction over the finite decoder depth, for every multi-index $\alpha$ and every parameter multi-index $\beta$ satisfying
\begin{equation*}
|\alpha|\le q,
\qquad
|\beta|\le2,
\end{equation*}
there exists $C_{\alpha,\beta}<\infty$ such that
\begin{equation*}
\sup_{(x,t)\in Q,\ z\in\cK_{\mathrm D}}
\left\|
\partial_z^{\beta}
\partial_{x,t}^{\alpha}
\widehat u_z(x,t;\mu^*)
\right\|_2
\le
C_{\alpha,\beta}.
\end{equation*}
By the regularity of $\cN_{\mu^*}$, $R_{\mathrm D}\in C^2(\cK_{\mathrm D})$, and $R_{\mathrm D}$, its Jacobian
\begin{equation*}
J_{\mathrm D}(z):=\nabla_zR_{\mathrm D}(z),
\end{equation*}
and all component Hessians $\nabla_z^2R_{\mathrm D,i}$ are bounded on $\cK_{\mathrm D}$.

Since
\begin{equation*}
\nabla \cL_{\mathrm D}(z)
=
2J_{\mathrm D}(z)^\top R_{\mathrm D}(z),
\end{equation*}
and
\begin{equation*}
\nabla^2\cL_{\mathrm D}(z)
=
2J_{\mathrm D}(z)^\top J_{\mathrm D}(z)
+
2
\sum_{i=1}^{N_R}
R_{\mathrm D,i}(z)
\nabla_z^2R_{\mathrm D,i}(z),
\end{equation*}
where $N_R:=\dim R_{\mathrm D}$, compactness gives
\begin{equation*}
\sup_{z\in\cK_{\mathrm D}}
\|\nabla^2\cL_{\mathrm D}(z)\|_2
\le
L_{\mathrm D}
<
\infty.
\end{equation*}
Thus the gradient Lipschitz property in \cref{eq:lipschitz_gradient_bound} follows.
\end{proof}

\subsection{Descent Convergence}

\begin{theorem}[Full-batch descent]
\label{thm:full_batch_descent}
Let Assumption~\ref{ass:moded_smoothness} hold. Consider
\begin{equation}
\label{eq:fullbatch_update}
z_{t+1}
=
z_t
-
\eta\nabla\cL_{\mathrm D}(z_t),
\qquad
0<\eta\le\frac1{L_{\mathrm D}},
\end{equation}
and assume $z_t\in\cK_{\mathrm D}$ for all $t$. Then
\begin{equation}
\label{eq:descent}
\cL_{\mathrm D}(z_{t+1})
\le
\cL_{\mathrm D}(z_t)
-
\frac{\eta}{2}
\|\nabla\cL_{\mathrm D}(z_t)\|_2^2.
\end{equation}
Moreover, for any $T\ge1$,
\begin{equation}
\min_{0\le t\le T-1}
\|\nabla\cL_{\mathrm D}(z_t)\|_2^2
\le
\frac{
2\left(
\cL_{\mathrm D}(z_0)-\cL_{\mathrm D}^{\star}
\right)
}{\eta T}.
\label{eq:fullbatch_stationarity}
\end{equation}
Consequently,
\begin{equation}
\lim_{t\to\infty}
\|\nabla\cL_{\mathrm D}(z_t)\|_2
=
0,
\label{eq:fullbatch_limit}
\end{equation}
and every accumulation point of $\{z_t\}_{t\ge0}$ is a first-order stationary point.
\Cref{eq:fullbatch_stationarity,eq:fullbatch_limit} summarize the finite-time and asymptotic stationarity guarantees.
\end{theorem}

\begin{proof}
By Lemma~\ref{lem:lipschitz_gradient},
\begin{equation*}
\cL_{\mathrm D}(z_{t+1})
\le
\cL_{\mathrm D}(z_t)
+
\left\langle
\nabla\cL_{\mathrm D}(z_t),
z_{t+1}-z_t
\right\rangle
+
\frac{L_{\mathrm D}}{2}
\|z_{t+1}-z_t\|_2^2.
\end{equation*}
Using \eqref{eq:fullbatch_update},
\begin{equation*}
\cL_{\mathrm D}(z_{t+1})
\le
\cL_{\mathrm D}(z_t)
-
\eta
\|\nabla\cL_{\mathrm D}(z_t)\|_2^2
+
\frac{L_{\mathrm D}\eta^2}{2}
\|\nabla\cL_{\mathrm D}(z_t)\|_2^2.
\end{equation*}
Since $\eta\le1/L_{\mathrm D}$,
\begin{equation*}
1-\frac{L_{\mathrm D}\eta}{2}\ge\frac12.
\end{equation*}
Thus \eqref{eq:descent} holds. Summing \eqref{eq:descent} over $t=0,\ldots,T-1$ gives
\begin{equation*}
\frac{\eta}{2}
\sum_{t=0}^{T-1}
\|\nabla\cL_{\mathrm D}(z_t)\|_2^2
\le
\cL_{\mathrm D}(z_0)-\cL_{\mathrm D}(z_T)
\le
\cL_{\mathrm D}(z_0)-\cL_{\mathrm D}^{\star}.
\end{equation*}
Hence
\begin{equation*}
\min_{0\le t\le T-1}
\|\nabla\cL_{\mathrm D}(z_t)\|_2^2
\le
\frac1T
\sum_{t=0}^{T-1}
\|\nabla\cL_{\mathrm D}(z_t)\|_2^2
\le
\frac{
2\left(
\cL_{\mathrm D}(z_0)-\cL_{\mathrm D}^{\star}
\right)
}{\eta T}.
\end{equation*}
Letting $T\to\infty$ yields
\begin{equation*}
\sum_{t=0}^{\infty}
\|\nabla\cL_{\mathrm D}(z_t)\|_2^2
<
\infty,
\end{equation*}
which implies $\|\nabla\cL_{\mathrm D}(z_t)\|_2\to0$. If $z_{t_j}\to \bar z$, then by continuity,
\begin{equation*}
\nabla\cL_{\mathrm D}(\bar z)
=
\lim_{j\to\infty}
\nabla\cL_{\mathrm D}(z_{t_j})
=
0.
\end{equation*}
\end{proof}

\begin{theorem}[Mini-batch convergence]
\label{thm:sgd_convergence}
Let Assumption~\ref{ass:moded_smoothness} hold. Let $g_t$ be a stochastic gradient estimator satisfying
\begin{equation}
\mathbb E[g_t\mid z_t]
=
\nabla\cL_{\mathrm D}(z_t),
\label{eq:sgd_unbiasedness}
\end{equation}
and
\begin{equation}
\mathbb E
\left[
\|g_t-\nabla\cL_{\mathrm D}(z_t)\|_2^2
\mid z_t
\right]
\le
\sigma_g^2.
\label{eq:sgd_variance}
\end{equation}
For
\begin{equation}
z_{t+1}=z_t-\eta g_t,
\qquad
0<\eta\le\frac1{L_{\mathrm D}},
\label{eq:sgd_update}
\end{equation}
we have
\begin{equation}
\frac1T
\sum_{t=0}^{T-1}
\mathbb E
\left[
\|\nabla\cL_{\mathrm D}(z_t)\|_2^2
\right]
\le
\frac{
2\left(
\cL_{\mathrm D}(z_0)-\cL_{\mathrm D}^{\star}
\right)
}{\eta T}
+
L_{\mathrm D}\eta\sigma_g^2.
\label{eq:sgd_stationarity}
\end{equation}
In particular, choosing $\eta=\mathcal O(T^{-1/2})$ gives
\begin{equation}
\frac1T
\sum_{t=0}^{T-1}
\mathbb E
\left[
\|\nabla\cL_{\mathrm D}(z_t)\|_2^2
\right]
=
\mathcal O(T^{-1/2}).
\label{eq:sgd_rate}
\end{equation}
\Cref{eq:sgd_unbiasedness,eq:sgd_variance,eq:sgd_update} are the stochastic assumptions and update rule, while \cref{eq:sgd_stationarity,eq:sgd_rate} give the mini-batch stationarity rate.
\end{theorem}

\begin{proof}
By Lemma~\ref{lem:lipschitz_gradient},
\begin{equation*}
\cL_{\mathrm D}(z_{t+1})
\le
\cL_{\mathrm D}(z_t)
-
\eta
\langle
\nabla\cL_{\mathrm D}(z_t),g_t
\rangle
+
\frac{L_{\mathrm D}\eta^2}{2}
\|g_t\|_2^2.
\end{equation*}
Taking conditional expectation and using unbiasedness,
\begin{equation*}
\mathbb E[
\cL_{\mathrm D}(z_{t+1})
\mid z_t]
\le
\cL_{\mathrm D}(z_t)
-
\eta
\|\nabla\cL_{\mathrm D}(z_t)\|_2^2
+
\frac{L_{\mathrm D}\eta^2}{2}
\mathbb E[
\|g_t\|_2^2
\mid z_t].
\end{equation*}
By the variance bound,
\begin{equation*}
\mathbb E[
\|g_t\|_2^2
\mid z_t]
\le
\|\nabla\cL_{\mathrm D}(z_t)\|_2^2
+
\sigma_g^2.
\end{equation*}
Therefore,
\begin{equation*}
\mathbb E[
\cL_{\mathrm D}(z_{t+1})
\mid z_t]
\le
\cL_{\mathrm D}(z_t)
-
\eta
\left(
1-\frac{L_{\mathrm D}\eta}{2}
\right)
\|\nabla\cL_{\mathrm D}(z_t)\|_2^2
+
\frac{L_{\mathrm D}\eta^2}{2}
\sigma_g^2.
\end{equation*}
Since $\eta\le1/L_{\mathrm D}$, $1-L_{\mathrm D}\eta/2\ge1/2$. Taking total expectation, summing over $t=0,\ldots,T-1$, and rearranging gives the assertion.
\end{proof}

\subsection{Zero-Degradation Initialization}

\begin{corollary}[Exact recovery at initialization]
\label{cor:zero_degradation}
Assume Assumption~\ref{ass:identity_off}. Let $z_0$ satisfy
\begin{equation}
\Phi_l^{(0)}=0,
\qquad
\tau_l^{(0)}=1,
\qquad
\Delta b_l^{(0)}=0,
\qquad
\xi_l^{(0)}=\xi_{l,0},
\qquad
\forall l\in\cA.
\label{eq:zero_degradation_initialization}
\end{equation}
Then
\begin{equation}
\widehat u_{z_0}(x,t;\mu^*)
=
\widehat u_{\mathrm{base}}(x,t;\mu^*),
\qquad
\forall (x,t)\in Q.
\label{eq:zero_degradation_identity}
\end{equation}
Consequently,
\begin{equation}
\cL_{\mathrm D}(z_0;\mu^*)
=
\cL_{\mathrm{base}}(\mu^*).
\label{eq:zero_degradation_loss}
\end{equation}
If the full-batch update \eqref{eq:fullbatch_update} is used with $0<\eta\le1/L_{\mathrm D}$, then
\begin{equation}
\cL_{\mathrm D}(z_t;\mu^*)
\le
\cL_{\mathrm{base}}(\mu^*),
\qquad
\forall t\ge0.
\label{eq:zero_degradation_monotone}
\end{equation}
\Cref{eq:zero_degradation_initialization,eq:zero_degradation_identity,eq:zero_degradation_loss,eq:zero_degradation_monotone} formalize the exact recovery and non-degradation guarantees.
\end{corollary}

\begin{proof}
Substituting the initialization into \eqref{eq:moded_layer} gives
\begin{equation*}
h_{z_0}^{(l)}
=
\sigma
\left(
W_0^{(l)}h_{z_0}^{(l-1)}
+
b_0^{(l)}
\right),
\qquad
\forall l\in\cA,
\end{equation*}
because
\begin{equation*}
U_l
\left[
0+(1-1)\Sigma_l
\right]
V_l^\top h=0,
\qquad
\cD_l(h;\xi_{l,0})=0.
\end{equation*}
Thus the MODED network coincides with the frozen foundational network at $t=0$. The loss identity follows immediately, and the monotone upper bound follows from Theorem~\ref{thm:full_batch_descent}.
\end{proof}

\subsection{Additional Descent Directions of MODED}

\begin{proposition}[Gradient enlargement at the MODE point]
\label{prop:gradient_enlargement}
Let $z=(\omega,\xi_0)$. Under Assumption~\ref{ass:identity_off},
\begin{equation}
R_{\mathrm D}(\omega,\xi_0)
=
R_{\mathrm M}(\omega).
\label{eq:residual_identity_mode_moded}
\end{equation}
Let
\begin{equation}
J_{\mathrm M}(\omega)
:=
\nabla_{\omega}R_{\mathrm M}(\omega),
\qquad
J_{\xi}(\omega,\xi_0)
:=
\nabla_{\xi}R_{\mathrm D}(\omega,\xi)\big|_{\xi=\xi_0}.
\label{eq:mode_moded_jacobians}
\end{equation}
Then
\begin{equation}
J_{\mathrm D}(\omega,\xi_0)
=
\left[
J_{\mathrm M}(\omega),
J_{\xi}(\omega,\xi_0)
\right],
\label{eq:mode_moded_block_jacobian}
\end{equation}
and
\begin{equation}
\|\nabla_z\cL_{\mathrm D}(\omega,\xi_0)\|_2^2
=
\|\nabla_{\omega}\cL_{\mathrm M}(\omega)\|_2^2
+
4
\|J_{\xi}(\omega,\xi_0)^\top R_{\mathrm M}(\omega)\|_2^2.
\label{eq:gradient_enlargement_identity}
\end{equation}
In particular,
\begin{equation}
\|\nabla_z\cL_{\mathrm D}(\omega,\xi_0)\|_2
\ge
\|\nabla_{\omega}\cL_{\mathrm M}(\omega)\|_2.
\label{eq:gradient_enlargement_bound}
\end{equation}
\Cref{eq:residual_identity_mode_moded,eq:mode_moded_jacobians,eq:mode_moded_block_jacobian} lead to the gradient enlargement relations in \cref{eq:gradient_enlargement_identity,eq:gradient_enlargement_bound}.
\end{proposition}

\begin{proof}
The residual identity follows from Assumption~\ref{ass:identity_off}. Differentiating
\begin{equation*}
\cL_{\mathrm D}(z)=\|R_{\mathrm D}(z)\|_2^2
\end{equation*}
gives
\begin{equation*}
\nabla_z\cL_{\mathrm D}(\omega,\xi_0)
=
2
J_{\mathrm D}(\omega,\xi_0)^\top
R_{\mathrm D}(\omega,\xi_0).
\end{equation*}
Using the block Jacobian representation,
\begin{equation*}
\nabla_z\cL_{\mathrm D}(\omega,\xi_0)
=
\begin{bmatrix}
2J_{\mathrm M}(\omega)^\top R_{\mathrm M}(\omega)
\\
2J_{\xi}(\omega,\xi_0)^\top R_{\mathrm M}(\omega)
\end{bmatrix}
=
\begin{bmatrix}
\nabla_{\omega}\cL_{\mathrm M}(\omega)
\\
2J_{\xi}(\omega,\xi_0)^\top R_{\mathrm M}(\omega)
\end{bmatrix}.
\end{equation*}
Taking squared norms gives the assertion.
\end{proof}

\subsection{Local Linear Convergence}

\begin{assumption}[Local PL geometry]
\label{ass:pl_moded}
There exist a neighborhood $\mathcal U_{\mathrm D}\subseteq\cK_{\mathrm D}$ and constants
\begin{equation}
\mu_{\mathrm{PL,D}}>0,
\qquad
\cL_{\mathrm D}^{\star}\ge0,
\label{eq:pl_constants}
\end{equation}
such that the MODED trajectory satisfies $z_t\in\mathcal U_{\mathrm D}$, and
\begin{equation}
\frac12
\|\nabla\cL_{\mathrm D}(z)\|_2^2
\ge
\mu_{\mathrm{PL,D}}
\left(
\cL_{\mathrm D}(z)-\cL_{\mathrm D}^{\star}
\right),
\qquad
\forall z\in\mathcal U_{\mathrm D}.
\label{eq:pl_condition}
\end{equation}
\Cref{eq:pl_constants,eq:pl_condition} specify the local PL geometry used for linear convergence.
\end{assumption}

\begin{theorem}[Linear convergence under PL]
\label{thm:pl_linear}
Let Assumptions~\ref{ass:moded_smoothness} and~\ref{ass:pl_moded} hold. For the full-batch update \eqref{eq:fullbatch_update} with
\begin{equation}
0<\eta\le\frac1{L_{\mathrm D}},
\label{eq:pl_stepsize}
\end{equation}
we have
\begin{equation}
\cL_{\mathrm D}(z_t)-\cL_{\mathrm D}^{\star}
\le
(1-\eta\mu_{\mathrm{PL,D}})^t
\left(
\cL_{\mathrm D}(z_0)-\cL_{\mathrm D}^{\star}
\right).
\label{eq:pl_linear_rate}
\end{equation}
If $z_0$ is initialized by Corollary~\ref{cor:zero_degradation}, then
\begin{equation}
\cL_{\mathrm D}(z_t)
\le
\cL_{\mathrm D}^{\star}
+
(1-\eta\mu_{\mathrm{PL,D}})^t
\left(
\cL_{\mathrm{base}}(\mu^*)
-
\cL_{\mathrm D}^{\star}
\right).
\label{eq:pl_zero_degradation_rate}
\end{equation}
Furthermore, under Proposition~\ref{prop:mode_subset_moded},
\begin{equation}
\cL_{\mathrm D}^{\star}
\le
\cL_{\mathrm M}^{\star}.
\label{eq:pl_floor_comparison}
\end{equation}
\Cref{eq:pl_stepsize,eq:pl_linear_rate,eq:pl_zero_degradation_rate,eq:pl_floor_comparison} provide the stepsize, linear rate, initialized rate, and MODE/MODED floor comparison.
\end{theorem}

\begin{proof}
By Theorem~\ref{thm:full_batch_descent},
\begin{equation*}
\cL_{\mathrm D}(z_{t+1})
-
\cL_{\mathrm D}^{\star}
\le
\cL_{\mathrm D}(z_t)
-
\cL_{\mathrm D}^{\star}
-
\frac{\eta}{2}
\|\nabla\cL_{\mathrm D}(z_t)\|_2^2.
\end{equation*}
Using the PL condition,
\begin{equation*}
\frac12
\|\nabla\cL_{\mathrm D}(z_t)\|_2^2
\ge
\mu_{\mathrm{PL,D}}
\left(
\cL_{\mathrm D}(z_t)-\cL_{\mathrm D}^{\star}
\right).
\end{equation*}
Therefore,
\begin{equation*}
\cL_{\mathrm D}(z_{t+1})
-
\cL_{\mathrm D}^{\star}
\le
(1-\eta\mu_{\mathrm{PL,D}})
\left(
\cL_{\mathrm D}(z_t)-\cL_{\mathrm D}^{\star}
\right).
\end{equation*}
Iterating proves the result. The zero-degradation form follows from Corollary~\ref{cor:zero_degradation}; the floor comparison follows from Proposition~\ref{prop:mode_subset_moded}.
\end{proof}

\subsection{PDE Solution Error Bound}

Define the continuous residual energy
\begin{align}
\cE_{\mathrm D}(z)
:=&\
\|\cN_{\mu^*}[\widehat u_z]\|_{L^2(Q)}^2
+
\|\widehat u_z(\cdot,0;\mu^*)-u_0(\cdot;\mu^*)\|_{L^2(\Omega)}^2
\notag\\
&+
\|\widehat u_z-u_b\|_{L^2(\Gamma)}^2.
\label{eq:pde_energy}
\end{align}

\begin{assumption}[PDE stability and quadrature consistency]
\label{ass:pde_stability}
Let $u^*$ be the exact solution of the target PDE at $\mu^*$. There exist constants
\begin{equation}
C_{\mathrm{stab}}>0,
\qquad
C_{\mathrm{quad}}>0,
\qquad
\varepsilon_{\mathrm{quad}}\ge0,
\label{eq:pde_stability_constants}
\end{equation}
such that, for all $z\in\cK_{\mathrm D}$,
\begin{equation}
\|\widehat u_z-u^*\|_{\cX}^2
\le
C_{\mathrm{stab}}
\cE_{\mathrm D}(z),
\label{eq:pde_stability_bound}
\end{equation}
and
\begin{equation}
\cE_{\mathrm D}(z)
\le
C_{\mathrm{quad}}
\cL_{\mathrm D}(z)
+
\varepsilon_{\mathrm{quad}}.
\label{eq:pde_quadrature_bound}
\end{equation}
\Cref{eq:pde_energy,eq:pde_stability_constants,eq:pde_stability_bound,eq:pde_quadrature_bound} connect empirical residual minimization with continuous PDE error.
\end{assumption}

\begin{theorem}[MODED solution error]
\label{thm:pde_error}
Let Assumptions~\ref{ass:moded_smoothness}, \ref{ass:pl_moded}, and~\ref{ass:pde_stability} hold. Then the full-batch MODED trajectory satisfies
\begin{equation}
\|\widehat u_{z_t}-u^*\|_{\cX}^2
\le
C_{\mathrm{stab}}
\left[
C_{\mathrm{quad}}
\left(
\cL_{\mathrm D}^{\star}
+
(1-\eta\mu_{\mathrm{PL,D}})^t
\left(
\cL_{\mathrm D}(z_0)
-
\cL_{\mathrm D}^{\star}
\right)
\right)
+
\varepsilon_{\mathrm{quad}}
\right].
\label{eq:pde_error_bound}
\end{equation}
If $z_0$ is initialized by Corollary~\ref{cor:zero_degradation}, then
\begin{equation}
\|\widehat u_{z_t}-u^*\|_{\cX}^2
\le
C_{\mathrm{stab}}
\left[
C_{\mathrm{quad}}
\left(
\cL_{\mathrm D}^{\star}
+
(1-\eta\mu_{\mathrm{PL,D}})^t
\left(
\cL_{\mathrm{base}}(\mu^*)
-
\cL_{\mathrm D}^{\star}
\right)
\right)
+
\varepsilon_{\mathrm{quad}}
\right].
\label{eq:pde_zero_degradation_error_bound}
\end{equation}
Moreover,
\begin{equation}
\cL_{\mathrm D}^{\star}
\le
\cL_{\mathrm M}^{\star}.
\label{eq:pde_floor_comparison}
\end{equation}
\Cref{eq:pde_error_bound,eq:pde_zero_degradation_error_bound,eq:pde_floor_comparison} state the solution-error transfer and the no-worse MODED residual floor.
\end{theorem}

\begin{proof}
By Assumption~\ref{ass:pde_stability},
\begin{equation*}
\|\widehat u_{z_t}-u^*\|_{\cX}^2
\le
C_{\mathrm{stab}}
\cE_{\mathrm D}(z_t)
\le
C_{\mathrm{stab}}
\left(
C_{\mathrm{quad}}\cL_{\mathrm D}(z_t)
+
\varepsilon_{\mathrm{quad}}
\right).
\end{equation*}
Substituting the PL estimate from Theorem~\ref{thm:pl_linear} yields the first bound. Substituting the zero-degradation identity yields the second bound. The floor comparison follows from Proposition~\ref{prop:mode_subset_moded}.
\end{proof}

\begin{remark}[Consequences]
The preceding results imply
\begin{equation*}
\mathrm{MODE}\subseteq\mathrm{MODED},
\qquad
\cL_{\mathrm D}^{\star}\le\cL_{\mathrm M}^{\star},
\end{equation*}
and, under Assumption~\ref{ass:pde_stability},
\begin{equation*}
\limsup_{t\to\infty}
\|\widehat u_{z_t}-u^*\|_{\cX}^2
\le
C_{\mathrm{stab}}
\left(
C_{\mathrm{quad}}\cL_{\mathrm D}^{\star}
+
\varepsilon_{\mathrm{quad}}
\right).
\end{equation*}
Thus MODED preserves the zero-degradation initialization and descent stability of MODE while attaining a no-worse residual and solution-error floor.
\end{remark}

\printcredits

\bibliographystyle{elsarticle-num}

\bibliography{cas-refs}

\end{document}